\documentclass[10pt]{article} 
\usepackage[accepted]{tmlr}

\usepackage{amsmath,amsfonts,bm}

\def\ceil#1{\lceil #1 \rceil}

\def\1{\bm{1}}

\DeclareMathAlphabet{\mathsfit}{\encodingdefault}{\sfdefault}{m}{sl}
\SetMathAlphabet{\mathsfit}{bold}{\encodingdefault}{\sfdefault}{bx}{n}

\DeclareMathOperator*{\argmin}{arg\,min}

\usepackage{amsmath, amssymb, amsthm}
\usepackage{graphics,amsopn,amstext,color,enumerate,bm}
\usepackage{color,float,epstopdf}
\usepackage{mathtools}
\usepackage{svg}
\usepackage{multirow}
\usepackage{graphicx}
\usepackage[normalem]{ulem}
\useunder{\uline}{\ul}{}
\usepackage[
  separate-uncertainty = true,
  multi-part-units = repeat
]{siunitx}

\newcommand{\bH}{{\mathbf{H}}}
\newcommand{\bM}{{\mathbf{M}}}
\newcommand{\bA}{{\mathbf{A}}}
\newcommand{\bB}{{\mathbf{B}}}
\newcommand{\bC}{{\mathbf{C}}}
\newcommand{\bD}{{\mathbf{D}}}
\newcommand{\bG}{{\mathbf{G}}}
\newcommand{\bI}{{\mathbf{I}}}
\newcommand{\bX}{{\mathbf{X}}}

\newcommand{\by}{{\mathbf{y}}}

\newcommand{\bx}{{\mathbf{x}}}
\newcommand{\bz}{{\mathbf{z}}}
\newcommand{\bb}{{\mathbf{b}}}
\newcommand{\bd}{{\mathbf{d}}}
\newcommand{\be}{{\mathbf{e}}}

\newcommand{\ba}{{\mathbf{a}}}
\newcommand{\bc}{{\mathbf{c}}}
\newcommand{\bw}{{\mathbf{w}}}
\newcommand{\bmu}{{\boldsymbol{\mu}}}
\newcommand{\boldeta}{{\boldsymbol{\eta}}}
\newcommand{\bgamma}{{\boldsymbol{\Gamma}}}

\newcommand{\btheta}{{\boldsymbol{\theta}}}

\newcommand{\st}{{\rm s.t.}}

\newtheorem{theorem}{Theorem}
\newtheorem{lemma}[theorem]{Lemma} 
 
\newtheorem{remark}[theorem]{Remark}
\newtheorem{corollary}[theorem]{Corollary}

\long\def\acks#1{\vskip 0.3in\noindent{\large\bf Acknowledgments}\vskip 0.04in
\noindent #1}

\usepackage{algpseudocode, algorithm}
\usepackage[noend]{algcompatible}

\usepackage{hyperref}
\usepackage{url}

\title{RIFLE: Imputation and Robust Inference from Low Order Marginals}

\author{\name Sina Baharlouei \email baharlou@usc.edu \\
      \addr
      University of Southern California
      \AND
      \name Kelechi Ogudu \email kogudu@usc.edu \\
      \addr University of Southern California
      \AND
      \name Sze-chuan Suen \email ssuen@usc.edu\\
      \addr University of Southern California
      \AND
      \name Meisam Razaviyayn \email razaviya@usc.edu\\
      \addr University of Southern California
      }

\def\month{09}  
\def\year{2023} 
\def\openreview{\url{https://openreview.net/forum?id=oud7Ny0KQy}} 

\begin{document}

\maketitle

\begin{abstract}
The ubiquity of missing values in real-world datasets poses a challenge for statistical inference and can prevent similar datasets from being analyzed in the same study, precluding many existing datasets from being used for new analyses. While an extensive collection of packages and algorithms have been developed for data imputation, the overwhelming majority perform poorly if there are many missing values and low sample sizes, which are unfortunately common characteristics in empirical data. Such low-accuracy estimations adversely affect the performance of downstream statistical models. We develop a statistical inference framework for \textit{regression and classification in the presence of missing data without imputation}. Our framework, RIFLE (Robust InFerence via Low-order moment Estimations), estimates low-order moments of the underlying data distribution with corresponding confidence intervals to learn a distributionally robust model. We specialize our framework to linear regression and normal discriminant analysis, and we provide convergence and performance guarantees. This framework can also be adapted to impute missing data. In numerical experiments, we compare RIFLE to several state-of-the-art approaches (including MICE, Amelia, MissForest, KNN-imputer, MIDA, and Mean Imputer) for imputation and inference in the presence of missing values. Our experiments demonstrate that RIFLE outperforms other benchmark algorithms when the percentage of missing values is high and/or when the number of data points is relatively small. RIFLE is publicly available at \url{https://github.com/optimization-for-data-driven-science/RIFLE}.
\end{abstract}

\section{Introduction}
Machine learning algorithms have shown promise when applied to various problems, including healthcare, finance, social data analysis, image processing, and speech recognition. However, this success mainly relied on the availability of large-scale, high-quality datasets, which may be scarce in many practical problems, especially in medical and health applications \citep{pedersen129785missing, sterne2009multiple, beaulieu2018characterizing}. 
Moreover, many experiments and datasets suffer from the small sample size in such applications. Despite the availability of a small number of data points in each study, an increasingly large number of datasets are publicly available. To fully and effectively utilize information on related research questions from diverse datasets, information across various datasets (e.g., different questionnaires from multiple hospitals with overlapping questions) must be combined in a reliable fashion. 
\begin{figure}[h]
\label{fig: Merging Datasets}
\centering
\vspace{-5mm}
\includegraphics[width=0.80\columnwidth]{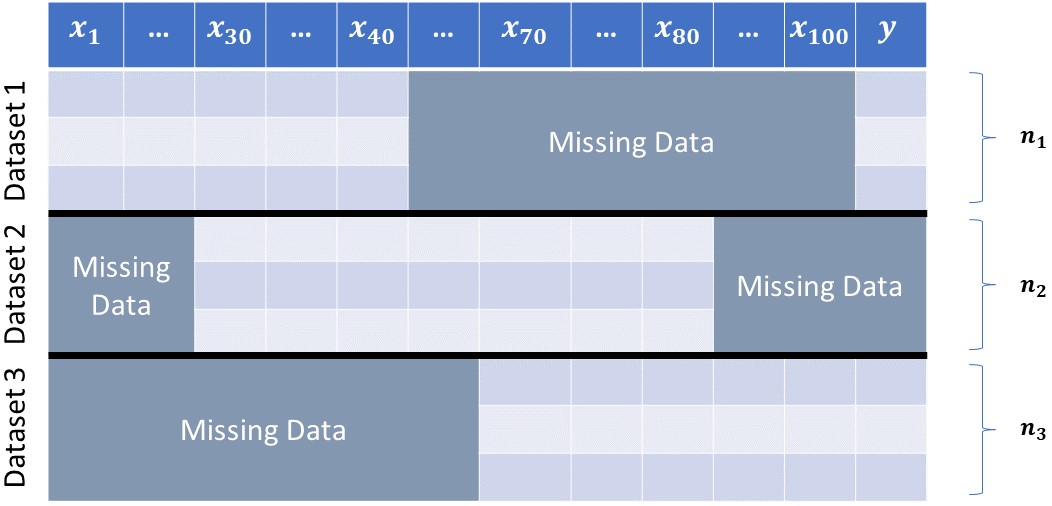}
\vspace{-3mm}
\caption{Consider the problem of predicting the trait $y$ from feature vector $(\bx_1, \dots, \bx_{100})$. Suppose that we have access to three data sets: The first dataset includes the measurements of $(\bx_1, \bx_2, \dots,\bx_{40}, y)$ for $n_1$ individuals. The second dataset collects data from another $n_2$ individuals by measuring $(\bx_{30}, \dots, \bx_{80})$ with no measurements of the target variable $y$ in it; and the third dataset contains the measurements from the variables $(\bx_{70}, \dots, \bx_{100}, y)$ for $n_3$ number of individuals. How one should learn the predictor $\hat{y} = h(\bx_1, \dots, \bx_{100})$ from these three datasets?}\label{fig: Missing_value_merged_sc}
\end{figure}

After integrating data from different studies, the obtained dataset can contain large blocks of missing values, as they may not share the same features (Figure~\ref{fig: Merging Datasets}).

There are three general approaches for handling missing values in statistical inference (classification and regression) tasks. A Na\"ive method is to remove the rows containing missing entries. However, such an approach is not an option when the percentage of missingness in a dataset is high. For instance, as demonstrated in Figure~\ref{fig: Missing_value_merged_sc}, the entire dataset will be discarded if we eliminate the rows with at least one missing entry. 

The most common methodology for handling missing values in a learning task is to impute them in a pre-processing stage. The general idea behind data imputation is that the missing values can be predicted using the available data entries and correlated features. Imputation algorithms cover a wide range of methods, including imputing missing entries with the columns means~\citet[Chapter~3]{little2019statistical} (or median),  least-square and linear regression-based methods~\citep{raghunathan2001multivariate, kim2005missing, zhang2008sequential, cai2006iterated, buuren2010mice}, matrix completion and expectation maximization approaches~\citet{dempster1977maximum, ghahramani1994supervised, honaker2011amelia}, KNN based~\citep{troyanskaya2001missing}, Tree based methods~\citep{stekhoven2012missforest, xia2017adjusted}, and methods using different neural network structures. Appendix~\ref{appendix: related_works} presents a comprehensive review of these methods.   

The imputation of data allows practitioners to run standard statistical algorithms requiring complete data. However, the prediction model's performance can be highly reliant on the accuracy of the imputer. High error rates in the prediction of missing values by the imputer can lead to the catastrophic performance of the downstream statistical methods executed on the imputed data. 

Another class of methods for inference in the presence of missing values relies on robust optimization over the uncertainty sets on missing entries. \citet{shivaswamy2006second} and~\citet{xu2009robustness} adopt robust optimization to learn the parameters of a support vector machine model. They consider uncertainty sets for the missing entries in the dataset and solve a min-max problem over those sets. The obtained classifiers are robust to the uncertainty of missing entries within the uncertainty regions. In contrast to the imputation-based approaches, the robust classification formulation does not carry the imputation error to the classification phase. However, finding appropriate intervals for each missing entry is challenging, and it is unclear how to determine the uncertainty range in many real datasets. Moreover, their proposed algorithms are limited to the SVM classifier. 

In this paper, we propose RIFLE (Robust InFerence via Low-order moment Estimations) for the direct inference of a target variable based on a set of features containing missing values. The proposed framework does not require the data to be imputed in a pre-processing stage. However, it can also be used as a pre-processing tool for imputing data. The main idea of the proposed framework is to estimate the first and second-order moments of the data and their confidence intervals by bootstrapping on the available data matrix entries. Then, RIFLE finds the optimal parameters of the statistical model for the worst-case distribution with the low-order moments (mean and variance) within the estimated confidence intervals (See Figure~\ref{fig: RIFLE_procedure}). Compared to~\citet{shivaswamy2006second, xu2009robustness}, we estimate uncertainty regions for the low-order marginals using the Bootstrap technique. Furthermore, our framework is not restricted to any particular machine learning model, such as support vector machines~\citep{xu2009robustness}. 
\begin{figure}
\centering
\vspace{-8mm}
\includegraphics[width=1.00\columnwidth]{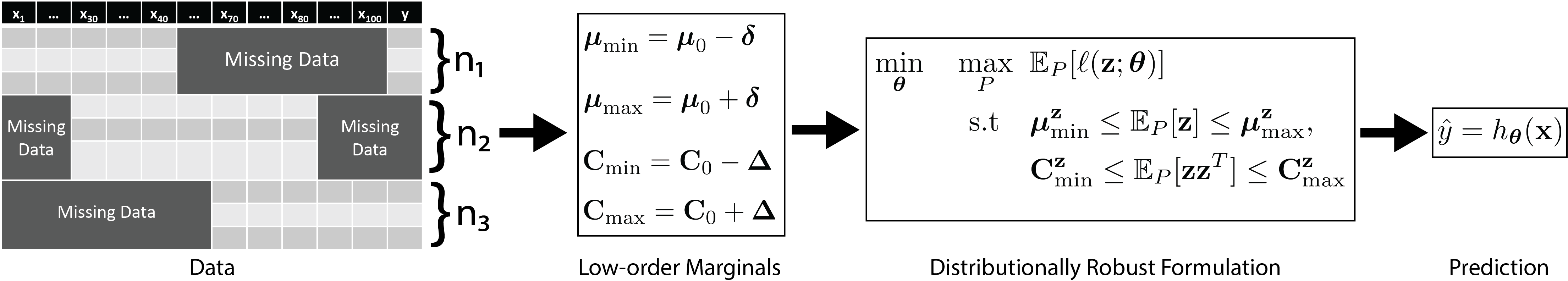}
\vspace{-6mm}
\caption{Prediction of the target variable without imputation. RIFLE estimates confidence intervals for low-order (first and second-order) marginals from the input data containing missing values. Then, it solves a distributionally robust problem over the set of all distributions whose low-order marginals are within the estimated confidence intervals.}\label{fig: RIFLE_procedure}
\end{figure}

\noindent \textbf{Contributions:} Our main contributions are as follows: 
\begin{enumerate}
    \item We present a distributionally robust optimization framework over the low-order marginals of the training data distribution for inference in the presence of missing values. The proposed framework does not require data imputation as a pre-processing stage. In Section~\ref{sec: robust_linear_regression} and Section~\ref{sec: robust_classification}, we specialize the framework to ridge regression and classification models as two case studies respectively. The proposed framework provides a novel strategy for inference in the presence of missing data, especially for datasets with large proportions of missing values.
    
    \item We provide theoretical convergence guarantees and the iteration complexity analysis of the presented algorithms for robust formulations of ridge linear regression and normal discriminant analysis. Moreover, we show the consistency of the prediction under mild assumptions and analyze the asymptotic statistical properties of the solutions found by the algorithms.
    
    \item While the robust inference framework is primarily designed for direct statistical inference in the presence of missing values without performing data imputation, it can also be adopted as an imputation tool. To demonstrate the quality of the proposed imputer, we compare its performance with several widely-used imputation packages such as MICE~\citep{buuren2010mice}, Amelia~\citep{honaker2011amelia}, MissForest~\citep{stekhoven2012missforest}, KNN-Imputer~\citep{troyanskaya2001missing}, MIDA~\citep{gondara2018mida}, GAIN~\citep{yoon2018gain} on real and synthetic datasets. Generally speaking, our method outperforms all of the mentioned packages when the number of missing entries is large.
\end{enumerate}

\section{Robust Inference via Estimating Low-order Moments} \label{sec: robust_inference_framework}
RIFLE is based on a distributionally robust optimization (DRO) framework over low-order marginals. Assume that $(\bx, y) \in \mathbb{R}^d \times \mathbb{R}$ follows a joint probability distribution~$P^*$. 
A standard approach for predicting the target variable $y$ given the input vector $\bx$ is to find the parameter~$\btheta$ that minimizes the population risk with respect to a given loss function $\ell$:
\begin{equation}
\label{eq: population_risk}
    \min_{\btheta} \quad \mathbb{E}_{(\bx, y) \sim P^*} \Big[ \ell \Big(\bx, y ; \boldsymbol{\theta} \Big) \Big].
\end{equation}
Since the underlying distribution of data is rarely available in practice, the above problem cannot be directly solved. The most common approach for approximating~\eqref{eq: population_risk} is to minimize the empirical risk with respect to $n$ given i.i.d samples $(\bx_1, y_1), \dots, (\bx_n, y_n)$ drawn from the joint distribution $P^*$:
\begin{equation*}
    \min_{\boldsymbol{\theta}} \quad \frac{1}{n} \sum_{i=1}^{n} \ell(\bx_i, y_i; \boldsymbol{\theta}).
\end{equation*}
The above empirical risk formulation assumes that all entries of $\bx_i$ and $y_i$ are available. Thus, to utilize the empirical risk minimization (ERM) framework in the presence of missing values, one can either remove or impute the missing data points in a pre-processing stage. Training via robust optimization is a natural alternative in the presence of missing data. \citet{shivaswamy2006second, xu2009robustness} suggest the following optimization problem that minimizes the loss function for the worst-case scenario over the defined uncertainty sets per data points:
\begin{equation}\label{eq.RobustUncertainty}
    \min_{\btheta} \;\;\max_{\{\boldsymbol{\delta}_i \in \mathcal{N}_i\}_{i=1}^n}\;\; \frac{1}{n}\sum_{i=1}^{n} \ell (\bx_i - \boldsymbol{\delta}_i, y_i;\btheta),
\end{equation}
where $\mathcal{N}_i$ represents the uncertainty region of data point~$i$. \citet{shivaswamy2006second} obtains the uncertainty sets by assuming a known distribution on the missing entries of datasets. The main issue in their approach is that the constraints defined on data points are totally uncorrelated. \citet{xu2009robustness} on the other hand defines $\mathcal{N}_i$ as a ``box'' constraint around the data point~$i$ such that they can be linearly correlated. For this specific case, they show that solving the corresponding robust optimization problem is equivalent to minimizing a regularized reformulation of the original loss function. Such an approach has several limitations: First, it can only handle a few special cases (SVM loss with linearly correlated perturbations on data points). Furthermore, \citet{xu2009robustness} is primarily designed for handling outliers and contaminated data. Thus, they do not offer any mechanism for the initial estimation of $\bx_i$ when several vector entries are missing. In this work, we instead take a \textit{distributionally robust} approach by considering uncertainty on the data distribution instead of defining an uncertainty set for each data point.  
In particular, we aim to fit the best parameters of a statistical learning model for the worst distribution in a given uncertainty set by solving the following: 
\begin{equation}\label{eq.DRO}
\quad \underset{\btheta}{\min} \quad \underset{P \in \mathcal{P}}{\max} \quad \mathbb{E}_{(\bx, y) \sim P} [\ell(\bx , y;\btheta)],
\end{equation}
where $\mathcal{P}$ is an uncertainty set over the underlying distribution of data. 
A \underline{key observation} is that defining the uncertainty set~$\mathcal{P}$ in~\eqref{eq.DRO} is easier and computationally more efficient than defining the uncertainty sets~$\{\mathcal{N}_i\}_{i=1}^n$ in~\eqref{eq.RobustUncertainty}. In particular, the uncertainty set~$\mathcal{P}$ can be obtained naturally by estimating low-order moments of data distribution \underline{using only available entries}. To explain this idea and to simplify the notations, let $\bz = (\bx, y)$,  $
 \bar{\bmu}^{\bz} \triangleq \mathbb{E}[\bz]$, and $ \bar{\bC}^{\bz} \triangleq \mathbb{E}[\bz \bz^T].$
While $\bar{\bmu}^{\bz}$ and $\bar{\bC}^{\bz}$ are typically not known exactly, one can estimate them (within certain confidence intervals) from the available data by simply ignoring missing entries (assuming the missing value pattern is completely at random, e.g., MCAR). Moreover, we can estimate the confidence intervals via bootstrapping. Particularly, we can estimate $\bmu_{\min}^{\bz},\bmu_{\max}^{\bz}, \bC^{\bz}_{\min}$, and $\bC^{\bz}_{\max}$ from data such that  $\bmu^{\bz}_{\min} \leq  \bar{\bmu}^{\bz} \leq \bmu^{\bz}_{\max}$ and $\bC^{\bz}_{\min} \leq \bar{\bC}^{\bz} \leq \bC_{\max}^{\bz}$ with high probability (where the inequalities for matrices and vectors denote component-wise relations). In Appendix~\ref{appendix: estimating_low_orders}, we show how a bootstrapping strategy can be used to obtain the confidence intervals described above. Given these estimated confidence intervals from data, \eqref{eq.DRO} can be reformulated as
\begin{equation}
\label{eq: general_framework}
\begin{split}
\min_{\btheta} \quad \max_{P} \;\; &\mathbb{E}_{ P} [\ell(\bz;\btheta)] \\
\st \quad & \bmu_{\min}^{\bz} \leq \mathbb{E}_{P} [\bz] \leq \bmu_{\max}^{\bz},\\
&  \bC_{\min}^{\bz} \leq \mathbb{E}_{P} [\bz\bz^T]  \leq \bC_{\max}^{\bz}.
\end{split}
\end{equation}
\citet{gao2017distributionally} utilize the distributionally robust optimization as~\eqref{eq.DRO} over the set of positive semi-definite (PSD) cones for robust inference under uncertainty. While their formulation considers $\ell_2$ balls for the constraints on low order moments of the data, we use $\ell_{\infty}$ constraints that are computationally more natural in the presence of missing entries when combined with bootstrapping. Furthermore, while it can be applied to general convex losses, their method relies on the ellipsoid and the existence of oracles for performing the steps of the ellipsoid method, which is not applicable in modern high-dimensional problems. Moreover, they assume concavity in data (the existence of some oracle to return the worst-case data points) that is practically unavailable even in convex loss functions (including linear regression and normal discriminant analysis studied in our work).

In Section~\ref{sec: robust_linear_regression}, we study the proposed distributionally robust framework described in~\eqref{eq: general_framework} for the ridge linear regression. We design efficient first-order convergent algorithms to solve the problem and show how we can use the algorithms for both inference and imputation in the presence of missing values. Further, in Appendix~\ref{appendix: robust_classification}, we study the proposed distributionally robust framework for the classification problems under the normality assumption of features. In particular, we show how Framework~\eqref{eq: general_framework} can be specialized to the robust normal discriminant analysis in the presence of missing values.
\section{Robust Linear Regression in the Presence of Missing Values} \label{sec: robust_linear_regression}
Let us specialize our framework to the ridge linear regression model. In the absence of missing data, ridge regression finds optimal regressor parameter $\btheta$ by solving
\begin{equation*}
\quad \underset{\btheta}{\min} \quad \Vert \bX \btheta - \by \Vert_2^2 + \lambda \| \btheta \|_2^2,  
\end{equation*}
or equivalently by solving:
\begin{equation} 
\min_{\btheta} \quad \btheta^T  \bX^T \bX \btheta - 2 \btheta^T \bX^T \by + \lambda \| \btheta\|_2^2.
\label{eq: linear_regression}
\end{equation}
Thus, having the second-order moments of the data $\bC = \bX^T \bX$ and $\bb = \bX^T \by$ is sufficient for finding the optimal solution. In other words, it suffices to compute the inner product of any two column vectors $\ba_i$, $\ba_j$ of $\bX$, and the inner product of any column $\ba_i$ of $\bX$ with vector $\by$. Since the matrix $\bX$ and vector $\by$ are not fully observed due to the existence of missing values, one can use the available data (see ~\eqref{eq: point_estimator} for details) to compute the point estimators $\bC_0$ and $\bb_0$. These point estimators can be highly inaccurate, especially when the number of non-missing rows for two given columns is small. In addition, if the pattern of missing entries does not follow the MCAR assumption, the point estimators are not unbiased estimators of $\bC$ and $\bb$.

\subsection{A Distributionally Robust Formulation of Linear Regression}
As we mentioned above, to solve the linear regression problem, we only need to estimate the second-order moments of the data ($\bX^T \bX$ and $\bX^T \by$). Thus, the distributionally robust formulation described in~\eqref{eq: general_framework} is equivalent to the following optimization problem for the linear regression model: 
\begin{equation}\label{robust_linear_regression}
\arraycolsep=1.4pt\def\arraystretch{1.3}
\begin{array}{lll}
    \displaystyle{\min_{\btheta} \quad \max_{\bC, \bb}} \quad \: &  \displaystyle{\btheta^T \bC \btheta - 2\bb^T \btheta + \lambda \| \btheta \|_2^2} \\
    \st & \bC_{0} - c\boldsymbol{\Delta} \leq \bC \leq \bC_0 + c\boldsymbol{\Delta}, \\
    & \bb_0 - c\boldsymbol{\delta} \leq \bb \leq \bb_0 + c\boldsymbol{\delta}, \\
    & \bC \succeq 0,
\end{array}
\end{equation}
where the last constraint guarantees that the covariance matrix is positive and semi-definite. We dicuss the procedure of estimating the confidence intervals ($\bb_0, \bC_0, \boldsymbol{\delta},$ and $\boldsymbol{\Delta}$) in Appendix~\ref{appendix: estimating_low_orders}.

\subsection{RIFLE for Ridge Linear Regression}
Since the objective function in~\eqref{robust_linear_regression} is convex in $\btheta$ (ridge regression) and concave in $\bb$ and $\bC$ (linear), the minimization and maximization sub-problems are interchangeable~\citep{sion1958general}. Thus, we can equivalently rewrite Problem~\eqref{robust_linear_regression} as:
\begin{equation}\label{robust_linear_regression2}
\vspace{-1mm}
\arraycolsep=1.4pt\def\arraystretch{1.3}
\begin{array}{lll}
    \displaystyle{\max_{\bC, \bb}} \quad \: & g(\bC, \bb) \\
    \st & \bC_{0} - c\boldsymbol{\Delta} \leq \bC \leq \bC_{0} + c\boldsymbol{\Delta}, \\
    & \bb_{0} - c\boldsymbol{\delta}\leq \bb \leq \bb_{0} + c\boldsymbol{\delta}, \\
    & \bC \succeq 0,
\end{array}
\end{equation}
where $g(\bb, \bC) = \min_{\btheta} \btheta^T \bC \btheta - 2 \bb^T \btheta + \lambda \| \btheta \|^2$. Function $g$ can be computed in closed-form given any pair of $(\bC, \bb)$ by setting $\btheta = (\bC+\lambda \bI)^{-1}\bb$. Thus, using Danskin's Theorem~\citep{danskin2012theory}, we can apply projected gradient ascent to function $g$ to find an optimal solution of~\eqref{robust_linear_regression2} as described in Algorithm~\ref{alg: Min_Max_Regression}. At each iteration of the algorithm, we first perform one step of projected gradient ascent on matrix $\bC$ and vector $\bb$; then we update $\btheta$ in closed-form for the obtained $\bC$ and $\bb$. We initialize $\bC$ and $\bb$ using entriwise point estimation on the available rows (see  Equation~\eqref{eq: point_estimator} in  Appendix~\ref{appendix: estimating_low_orders}).
\begin{algorithm}
	\caption{RIFLE for Ridge Linear Regression in the Presence of Missing Values} 
	\label{alg: Min_Max_Regression}
	\begin{algorithmic}[1]
	    \State \textbf{Input}: $\bC_0, \bb_0, \boldsymbol{\Delta}, \boldsymbol{\delta}, T$
        \State \textbf{Initialize}: $\bC = \bC_0, \bb = \bb_0$. 
        
        \FOR {$i = 1, \ldots, T$}
        
        \State Update $\bC = \Pi_{\boldsymbol{\Delta}+}\left[\bC + \alpha \btheta \btheta^T \right]$
        
        \State Update $\bb = \Pi_{\boldsymbol{\delta}}(\bb - 2\alpha\btheta)$
        \State Set $\btheta = (\bC+\lambda \bI)^{-1}\bb$ 
        \ENDFOR
	\end{algorithmic}
\end{algorithm}
The projection of $\bb$ to the box constraint $\bb_0 - c\boldsymbol{\delta} \leq \bb \leq \bb_0 + c\boldsymbol{\delta}$ can be done entriwise and has the following closed-form
\begin{equation*}
    \Pi_{\delta}(\bb_i) = \begin{cases}
        \bb_i  & \textrm{if } \quad \bb_{0i} - c\boldsymbol{\delta}_i\leq \bb_i \leq \bb_{0i} + c\boldsymbol{\delta}_i, \\
        \bb_{0i} - c\boldsymbol{\delta}_i  & \textrm{if } \quad \bb_i < \bb_{0i} - c\boldsymbol{\delta}_i, \\ 
        \bb_{0i} + c\boldsymbol{\delta}_i  & \textrm{if } \quad \bb_{0i} + c\boldsymbol{\delta}_i < \bb_i.
    \end{cases}
\end{equation*}
\begin{theorem}\label{thm: Robust_Regression_Convergence}
Let $(\tilde{\btheta}, \tilde{\bC}, \tilde{\bb})$ be the optimal solution of~\eqref{robust_linear_regression},  $\btheta^{*}(\bb, \bC) = \argmin_{\btheta} \btheta^T \bC \btheta - 2 \bb^T \btheta + \lambda \| \btheta \|^2$, and $D = \|\bC_0 - \tilde{\bC}\|_F^2 + \| \bb_0 - \tilde{\bb}\|_2^2$. Assume that for any given $\bb$ and $\bC$, within the uncertainty (constraint) sets described in~\eqref{robust_linear_regression},  $\|\btheta^{*}(\bb, \bC)\| \leq \tau$. Then Algorithm~\ref{alg: Min_Max_Regression} computes an $\epsilon$-optimal solution of the objective function in \eqref{robust_linear_regression2} in  ${\cal O}\Big(\frac{D(\tau+1)^2}{\lambda \epsilon}\Big)$ iterations.
\end{theorem}
\vspace{-2mm}
\begin{proof}
The proof is relegated to Appendix~\ref{appendix: proofs}.
\end{proof}
In Appendix~\ref{appendix: Nesterov_Acceleration}, we show how using the acceleration method of Nesterov can improve the convergence rate of Algorithm~\ref{alg: Min_Max_Regression} to ${\cal O}\Big(\sqrt{\frac{D(\tau+1)^2}{\epsilon \lambda}}\Big)$. A technical issue of Algorithm~\ref{alg: Min_Max_Regression} and its accelerated version presented in Appendix~\ref{appendix: Nesterov_Acceleration} is that projection of $\bC$ to the intersection of box constraints and the set of positive semidefinite matrices ($\Pi_{\boldsymbol{\Delta}+}\left[\bC\right]$) is challenging and cannot be done in closed-form. In the implementation of Algorithm~\ref{alg: Min_Max_Regression}, we relax the problem by removing the PSD constraint on $\bC$ to avoid this complexity and time-consuming singular value decomposition at each iteration. This relaxation does not drastically change the algorithm's performance, as our experiments show in Section~\ref{sec: numerical_results}. A more systematic approach is to write the dual problem of the maximization problem and handle the resulting constrained minimization problem with the Alternating Direction Method of Multipliers (ADMM). The detailed procedure of such an approach can be found in Appendix~\ref{appendix: admm}. All these algorithms are provably convergent to the optimal points of Problem~\eqref{robust_linear_regression}. In addition to theoretical convergence, we have numerically evaluated the convergence of resulting algorithms in Appendix~\ref{appendix: convergence}. Further, the proposed algorithms are \textbf{consistent}, as discussed in Appendix~\ref{appendix: consistency}.

\subsection{Performance Guarantees for RIFLE}
Thus far, we have discussed how to efficiently solve the robust linear regression problem in the presence of missing values. A natural question in this context is the statistical performance of the obtained optimal solution in the previous section on the unseen test data points. Theorem~\ref{thm: regression_performance} answers this question from two perspectives: Assuming that the missing values are distributed completely at random, our estimators are consistent. Moreover, for the finite case, Theorem~\ref{thm: regression_performance} part (b) states that with the proper choice of confidence intervals, with high probability, the test loss of the obtained solution is bounded by the training loss of the estimator. Note that the results regarding the performance of the robust estimator generally hold for MCAR missing pattern. However, we perform several experiments on datasets with MNAR patterns to show how RIFLE works in practice on such datasets in Section~\ref{sec: numerical_results}.
\begin{theorem}\label{thm: regression_performance}
Assume the data domain is bounded and that the missing pattern of the data follows MCAR. Let $\bX^{n \times d}$, $\by$ be the training data  drawn i.i.d. from the ground-truth distribution $P^*$ with low-order moments~$\bC^*$ and $\bb^*$. Further, assume that each entry of $\bX$ and $\by$ is missing with probability $p<1$. Let $(\tilde{\btheta}_n, \tilde{\bC}_n, \tilde{\bb}_n)$ be the solution of Problem~\eqref{robust_linear_regression}. 

\textbf{(a)} \textbf{Consistency of the Covariance Estimator}: As the number of data points goes to infinity, the estimated low-order marginals converge to the ground-truth values, almost surely. More precisely, 
\begin{gather}
    \lim_{n \rightarrow \infty} \tilde{\bC}_n = \mathbb{E}_{P^{*}}[\bx \bx^T], \quad a.s., \\
    \lim_{n \rightarrow \infty} \tilde{\bb}_n = \mathbb{E}_{P^{*}}[\bx y],\quad a.s.
\end{gather}
\textbf{(b)} 
Defining 
\begin{align*}
    & \normalfont L_{\textrm{train}}(\tilde{\btheta}_n) = \displaystyle{\tilde{\btheta}_n^T \tilde{\bC}_n \tilde{\btheta}_n - 2\tilde{\bb}_n \tilde{\btheta}_n + \lambda \| \tilde{\btheta}_n \|_2^2} \\  
    & \normalfont L_{\textrm{test}}(\tilde{\btheta}_n) =  \displaystyle{\tilde{\btheta}_n^T \bC^* \tilde{\btheta}_n - 2\bb^{*T} \tilde{\btheta}_n + \lambda \| \tilde{\btheta}_n \|_2^2},
\end{align*}
where $\bC^* = \mathbb{E}_{(\bx, y) \sim P^*}[\bx \bx^T]$ and $\bb^* = \mathbb{E}_{(\bx, y) \sim P^*}[\bx y]$ are the ground-truth second-order moments. Given $\normalfont V = \max_{i, j} \textrm{Var}(X_i X_j)$ (maximum variance of pairwise feature products), with the probability of at least $1 - \frac{d^2 V}{2c^2 \boldsymbol{\Delta}^2 n(1-p)}$, we have:
\begin{equation}
\normalfont
    L_{\textrm{test}}(\tilde{\btheta}) \leq L_{\textrm{train}}(\tilde{\btheta}),   
\end{equation}
where $\Delta = \min \{\Delta_{ij}\}$  and $c$ is the hyper-parameter for controlling the size of the confidence intervals as presented in~\eqref{robust_linear_regression} \end{theorem}
\begin{proof}
The proof is relegated to Appendix~\ref{appendix: proofs}.
\end{proof}

\subsection{Imputation of Missing Values and Going Beyond Linear Regression}
RIFLE can be used for imputing missing data. To this end, we impute different features of a given dataset independently. More precisely, to impute each feature containing missing values, we consider it as a target variable $\by$ and the rest of the features as the input $\bX$ in our methodology. Then, we train a model to predict the feature $\by$ given $\bX$ via Algorihm~\ref{alg: Min_Max_Regression} (or its ADMM version, Algorithm~\ref{alg: solving_dual_admm}, in the appendix). 
Let the obtained optimal solutions be $\bC^*, \bb^*,$ and $\btheta^*$. For a given missing entry, we can use $\btheta^*$ only if all other features in the row of that missing entry are available. 
However, that is not usually the case in practice, as each row can contain more than one missing entry. Therefore, one can learn a separate model for each missing pattern in the dataset. Let us clarify this point through the example in Figure~\ref{fig: Merging Datasets}. In this example, we have three different missing patterns (one missing pattern for each dataset). For missing entries in Dataset~1, the first forty features are available. 
Let $\mathbf{r}_j$ denote the vector of the first $40$ features in row $j$. Assume that we aim to impute entry $i \in \{41, \dots, 100\}$ in row $j$ where $i$ denoted by $x_{ji}$. To this end, we restrict $\bX$ to the first $40$ features. Moreover, we consider $y = x_i$ as the target variable. Then, we run Algorithm~\ref{alg: Min_Max_Regression} on $\bX$ and $y$ to obtain the optimal $\bC^*$, $\bb^*_{i}$, and $\btheta^*_{i}$. Consequently, we impute $x_{ji}$ as follows:
\begin{equation*}
    x_{ji} = \mathbf{r}_{j}^T \btheta^*_{i}
\end{equation*}
We can use the same methodology for imputing missing entries in each feature for missing patterns in Dataset~2 and Dataset~3. While this approach is reasonable for the missing pattern observed in Figure~\ref{fig: Merging Datasets}, in many practical problems, different rows can have distinct missing patterns. Thus, in the worst case, Algorithm~\ref{alg: Min_Max_Regression} must be executed once for each missing entry. Such an approach is computationally expensive and might be infeasible in large-scale datasets containing large amounts of missing entries. Alternatively, one can perform Algorithm~\ref{alg: Min_Max_Regression} only once to obtain $\bC^*$ and $\bb^*$ (considered the ``worst-case/pessimistic'' estimation of the moments). 
Then to impute each missing entry, $\bC^*$ and $\bb^*$ are restricted to the features available in that missing entry's row. Having the restricted $\bC^*$ and $\bb^*$, the regressor~$\btheta^{*}$ can be obtained in closed-form (line 6 in Algorithm~\ref{alg: Min_Max_Regression}). In this approach, we perform algorithm~\ref{alg: Min_Max_Regression} once and find the optimal $\btheta^*$ for each missing entry based on the estimated  $\bC^*$ and $\bb^*$. This approach can lead to sub-optimal solutions compared to the former approach, but it is much faster and more scalable.

\noindent\textbf{Beyond Linear Regression:} While the developed methods are primarily designed for ridge linear regression, one can apply non-linear transformations (kernels) to obtain models beyond linear. In Appendix~\ref{appendix: qrifle}, we show how to extend the developed algorithms to quadratic models. The RIFLE framework applied to the quadratically transformed data is called \textbf{QRIFLE}.
\section{Robust Classification Framework}
\label{sec: robust_classification}
In this section, we study the proposed framework in~\eqref{eq: general_framework} for the classification tasks in the presence of missing values. Since the target variable $y \in \mathcal{Y} = \{1, \dots, M\}$ takes discrete values in classification tasks, we consider the uncertainty sets over the data's first- and second-order marginals given each target value (label) separately. Therefore, the distributionally robust classification over low-order marginals can be described as:
\begin{equation}
\label{eq: Min_Max_Classification}
\begin{split}
     \min_{\mathbf{w}} \;\max_{P} \;\;& \mathbb{E}_{P}[\ell(\bx, y, \mathbf{w})]
     \\ \st \: \: &\bmu_{\textrm{min}, y} \: \leq  \: \mathbb{E}_P[\bx | y]\: \leq \: \bmu_{\textrm{max}, y} \quad \forall y \in \mathcal{Y} \\ 
     & \boldsymbol{\Sigma}_{\textrm{min}, y} \: \leq  \: \mathbb{E}_P[\bx \bx^T | y] \: \leq \: \boldsymbol{\Sigma}_{\textrm{max}, y} \quad \forall y \in \mathcal{Y}
\end{split}
\end{equation}
where $\boldsymbol{\mu}_{\min}, \boldsymbol{\mu}_{\max}, \boldsymbol{\Sigma}_{\min},$ and $\boldsymbol{\Sigma}_{\max}$ are the estimated confidence intervals for the first and second order of the data distribution. Unlike the robust linear regression task in Section~\ref{sec: robust_linear_regression}, the evaluation of the objective function in~\eqref{eq: Min_Max_Classification} might depend on higher-order marginals (beyond second-order) due to the nonlinearity of the loss function. As a result, Problem~\eqref{eq: Min_Max_Classification} is a non-convex non-concave intractable min-max optimization problem in general. For the sake of computational traceability, we restrict the distribution in the inner maximization problem to the set of normal distributions. In the following section, we specialize~\eqref{eq: Min_Max_Classification} to the quadratic discriminant analysis as a case study. The methodology can be extended to other popular classification algorithms, such as support vector machines and multi-layer neural networks.
\vspace{-2mm}
\subsection{Robust Quadratic Discriminant Analysis}
\vspace{-2mm}
Learning a logistic regression model on datasets containing missing values has been studied extensively in the literature~\citep{fung1989treatment, abonazel2018estimation}. Besides deleting missing values and imputation-based approaches, \citet{fung1989treatment} models the logistic regression task in the presence of missing values as a linear discriminant analysis problem where the underlying assumption is that the predictors follow normal distribution conditional on the labels. Mathematically speaking, they assume that the data points assigned to a specific label follow a Gaussian distribution, i.e., $\bx | y = i \sim N(\bmu_i, \boldsymbol{\Sigma})$. They use the available data to estimate the parameters of each Gaussian distribution. Therefore, the parameters of the logistic regression model can be assigned based on the estimated parameters of the Gaussian distributions for different classes. Similar to the linear regression case, the estimations of means and covariances are unbiased only when the data satisfies the MCAR condition. Moreover, when the number of data points in the dataset is small, the variance of the estimations can be very high. Thus, to train a logistic regression model that is robust to the percentage and different types of missing values, we specialize the general robust classification framework formulated in Equation~\eqref{eq: Min_Max_Classification} to the logistic regression model. Instead of considering a common covariance matrix for the conditional distributions of $\bx$ given labels $y$ (linear discriminant analysis), we assume a more general case where each conditional distribution has its own covariance matrix (quadratic discriminant analysis). 
 Assume that $\mathbf{x} | y \sim N(\bmu_y, \boldsymbol{\Sigma}_y)$ for $y = 0, 1$. We aim to find the optimal solution to the following problem:
\begin{equation}
\label{robust_logistic_regression}
\arraycolsep=1.4pt\def\arraystretch{1.3}
\begin{array}{lll}
    \displaystyle{\min_{\mathbf{w}} \max_{\boldsymbol{\mu}_0, \boldsymbol{\mu}_1, \boldsymbol{\Sigma}_0, \boldsymbol{\Sigma}_1}} \quad &  \displaystyle{\mathbb{E}_{\mathbf{x}| y=1 \sim N(\boldsymbol{\mu}_1, \boldsymbol{\Sigma}_1)} \Big[-\log \Big(\sigma(\mathbf{w}^T \mathbf{x}) \Big)\Big] \mathbb{P}(y = 1)} \: +  \\ & \displaystyle{\mathbb{E}_{\mathbf{x}| y=0 \sim N(\boldsymbol{\mu}_0, \boldsymbol{\Sigma}_0)} \Big[-\log \Big(1 - \sigma(\mathbf{w}^T \mathbf{x}) \Big)\Big]}\mathbb{P} (y = 0)  \\
    \st & \boldsymbol{\mu}_{\min_0} \leq \boldsymbol{\mu}_0 \leq \boldsymbol{\mu}_{\max_0}\\
    & \boldsymbol{\mu}_{\min_1} \leq \boldsymbol{\mu}_1 \leq \boldsymbol{\mu}_{\max_1}\\
    & \boldsymbol{\Sigma}_{\min_0} \leq \boldsymbol{\Sigma}_0 \leq \boldsymbol{\Sigma}_{\max_0}\\
    & \boldsymbol{\Sigma}_{\min_1} \leq \boldsymbol{\Sigma}_1 \leq \boldsymbol{\Sigma}_{\max_1}\\
\end{array}
\end{equation}
Where $\sigma(\bx) = 1/ \Big(1 + \exp(-\bx)\Big)$ is the sigmoid function.

To solve Problem~\eqref{robust_logistic_regression}, first, we focus on the scenario when the target variable has no missing values. In this case, each data point contributes to the estimation of either $(\bmu_1, \boldsymbol{\Sigma}_1)$ or $(\bmu_0, \boldsymbol{\Sigma}_0)$, depending on its label. 
Similar to the robust linear regression case, we can apply Algorithm~\ref{alg: estimating_confidence_interval_length} to estimate the confidence intervals for $\boldsymbol{\mu}_i, \boldsymbol{\Sigma}_i$ using data points whose target variable equals $i$ ($y = i$). 

Obviously, the objective function is convex in $\mathbf{w}$ since the logistic regression loss is convex, and the expectation of loss can be seen as a weighted summation, which is convex. Thus, fixing $\bmu, \boldsymbol{\Sigma}$ the outer minimization problem can be solved with respect to $\bw$ using standard first-order methods such as gradient descent.

Although the robust reformulation of logistic regression stated in~\eqref{robust_logistic_regression} is convex in $\bw$ and concave in $\bmu_0$ and $\bmu_1$, the inner maximization problem is intractable with respect to $\boldsymbol{\Sigma}_0$ and $\boldsymbol{\Sigma}_1$. We approximate Problem~\eqref{robust_logistic_regression} in the following manner:
\begin{equation}
\label{eq: robust_logistic_regression_finite}
\arraycolsep=1.4pt\def\arraystretch{1.3}
\begin{array}{lll}
    \displaystyle{\min_{\mathbf{w}} \max_{\boldsymbol{\mu}_0, \boldsymbol{\Sigma}_0, \boldsymbol{\mu}_1, \boldsymbol{\Sigma}_1}} \quad &  \displaystyle{\pi_1 \mathbb{E}_{\mathbf{x}|y=1 \sim N(\bmu_1, \boldsymbol{\Sigma}_1)} \Big[- \log \Big(\sigma(\mathbf{w}^T \mathbf{x}) \Big) \Big] + \pi_0 \mathbb{E}_{\mathbf{x}|y=0 \sim N(\bmu_0, \boldsymbol{\Sigma}_0)} \Big[- \log \Big(1 - \sigma (\mathbf{w}^T \mathbf{x})\Big) \Big]}, \\
    \st & \boldsymbol{\mu}_{\min_0} \leq \boldsymbol{\mu}_0 \leq \boldsymbol{\mu}_{\max_0}\\
    & \boldsymbol{\mu}_{\min_1} \leq \boldsymbol{\mu}_1 \leq \boldsymbol{\mu}_{\max_1}\\ &  \boldsymbol{\Sigma}_0 \in \{\boldsymbol{\Sigma}_{01}, \boldsymbol{\Sigma}_{02}, \dots, \boldsymbol{\Sigma}_{0k}\} \\   &  \boldsymbol{\Sigma}_1 \in \{\boldsymbol{\Sigma}_{11}, \boldsymbol{\Sigma}_{12}, \dots, \boldsymbol{\Sigma}_{1k}\},
\end{array}
\end{equation}
where $\pi_1 = \mathbb{P}(y = 1)$ and $\pi_0 = \mathbb{P}(y = 0)$. To compute optimal $\boldsymbol{\mu}_0$ and $\boldsymbol{\mu}_1$, we have:
\begin{equation}
\label{eq: mu_maximization}
    \max_{\bmu_1} \quad \mathbb{E}_{\mathbf{x} \sim N(\bmu_1, \boldsymbol{\Sigma}_1)} \Big[- \log \Big(\sigma(\mathbf{w}^T \mathbf{x}) \Big) \Big] 
    \quad \st \quad \boldsymbol{\mu}_{\min} \leq \boldsymbol{\bmu}_1 \leq \boldsymbol{\bmu}_{\max} 
\end{equation}

\begin{theorem}
\label{thm: mu_max}
Let $\ba[i]$ be the $i$-th element of vector $\ba$. The optimal solution of Problem~\eqref{eq: mu_maximization} has the following form:
\begin{equation}
\label{eq: mu_update}
    \bmu^*_{1} [i] = \begin{cases} \bmu_{\max}[i], \quad \quad  \mbox{if } \: \bw[i] \leq 0 \\ 
\bmu_{\min}[i], \quad \quad \mbox{if } \: \bw[i] > 0. \end{cases}
\end{equation}
\end{theorem}

Note that we relaxed~\eqref{robust_logistic_regression} by taking the maximization problem over a finite set of $\boldsymbol{\Sigma}$ estimations. We estimate each $\boldsymbol{\Sigma}$ by bootstrapping on the available data using Algorithm~\ref{alg: estimating_confidence_interval_length}. Define $f_{i} (\bw)$ as: 
\begin{equation}
\label{eq: f_i_definition}
\arraycolsep=1.4pt\def\arraystretch{1.3}
\begin{array}{lll}
    f_i(\bw) = \displaystyle{\pi_1 \mathbb{E}_{\mathbf{x} \sim N(\bmu^{*}_{1}, \boldsymbol{\Sigma}_{i1})} \Big[- \log \Big(\sigma(\mathbf{w}^T \mathbf{x}) \Big) \Big]}
\end{array}
\end{equation}
Similarly, we can define:
\begin{equation}
\label{eq: g_i_definiion}
\arraycolsep=1.4pt\def\arraystretch{1.3}
\begin{array}{lll}
    g_i(\bw) = \displaystyle{\pi_0 \mathbb{E}_{\mathbf{x} \sim N(\bmu^{*}_{0}, \boldsymbol{\Sigma}_{i0})} \Big[- \log \Big(1 - \sigma(\mathbf{w}^T \mathbf{x}) \Big) \Big]}
\end{array}
\end{equation}
Since the maximization problem is over a finite set, we can rewrite Problem~\eqref{eq: robust_logistic_regression_finite} as:
\begin{equation}
\label{eq: robust_logistic_regression_finite_linear_combination}
\arraycolsep=1.4pt\def\arraystretch{1.3}
\begin{array}{lll}
    \displaystyle{\min_{\mathbf{w}} \max_{i, j \in \{1, \dots, k\}}} \quad &  \displaystyle{f_i(\bw) + g_j (\bw)}  = \displaystyle{\min_{\mathbf{w}} \max_{p_1, \dots, p_k, q_1, \dots, q_k}} \quad &  \displaystyle{\sum_{i = 1}^{k} p_i f_i(\bw) + \sum_{j = 1}^{k} p_i g_j(\bw)} \\
    \st & \sum_{i=1}^{k} p_i = 1, \quad p_i \geq 0 \\
        & \sum_{j=1}^{k} q_j = 1, \quad q_j \geq 0
\end{array}
\end{equation}

Since the maximum of several functions is not necessarily smooth (differentiable), we add a quadratic regularization term to the maximization problem, accelerating the convergence rate \citep{nouiehed2019solving} as follows: 
\begin{equation}
\label{eq: robust_logistic_regression_finite_regularized}
\arraycolsep=1.4pt\def\arraystretch{1.3}
\begin{array}{lll}
    \displaystyle{\min_{\mathbf{w}} \max_{p_1, \dots, p_k, q_1, \dots, q_k}} \quad &  \displaystyle{\sum_{i = 1}^{k} p_i f_i(\bw)} - \delta \sum_{i=1}^{k} p_i ^2 + \sum_{j = 1}^{k} q_j g_j(\bw) - \delta \sum_{j=1}^{k} q_j ^2 \\
    \st & \sum_{i=1}^{k} p_i = 1, \quad p_i \geq 0 \\
        & \sum_{j=1}^{k} q_j = 1, \quad q_j \geq 0
\end{array}
\end{equation}
First, we show how to solve the inner maximization problem. Note that the $p_i$'s and $q_i$'s are independent. We show how to find optimal $p_i$'s. Optimizing with respect to $q_i$'s is similar. Since the maximization problem is a constrained quadratic program, we can write the Lagrangian function as follows:
\begin{equation}
\label{eq: robust_logistic_regression_finite_lagrangian}
\arraycolsep=1.4pt\def\arraystretch{1.3}
\begin{array}{lll}
    \displaystyle{\max_{p_1, \dots, p_k}} \quad &  \displaystyle{\sum_{i = 1}^{k} p_i f_i(\bw)} - \delta \sum_{i=1}^{k} p_i ^2 - \lambda (\sum_{i = 1}^{k} p_i - 1)\\
    \st \quad p_i \geq 0
\end{array}
\end{equation}
Having the optimal $\lambda$, the above problem has a closed-form solution with respect to each $p_i$, which can be written as:
\begin{equation*}
    p^{*}_i = \left[\frac{-\lambda + f_i}{2 \delta}\right]_{+}
\end{equation*}
Since $p^{*}_i$ is a non-increasing function with respect to $\lambda$, we can find the optimal value of $\lambda$ using the following bisection algorithm. Algorithm~\ref{alg: bisection} demonstrates how to find an $\epsilon$-optimal $\lambda$ and $p_i^{*}$'s efficiently using the bisection idea. 
\begin{algorithm}
	\caption{Finding the optimal $\lambda$ and $p_i$'s  using the bisection idea} 
	\label{alg: bisection}
	\begin{algorithmic}[1]
        \State \textbf{Initialize}: $\lambda_{\textrm{low}} = 0, \lambda_{\textrm{high}} = \max_{i} f_i, p_i = 0 \quad \forall i \in \{1, 2, \dots, k\}$. 
        
        \WHILE {$|\sum _{i=1}^ {n} p_k - 1| > \epsilon$}
            
            \State $\lambda = \frac{\lambda_{\textrm{low}} + \lambda_{\textrm{high}}}{2}$
                        
            \State Set $p_i = [\frac{-\lambda + f_i}{2\delta}]_{+} \quad \forall i \in \{1, 2, \dots, k\}$
            
            \IF {$\sum_{i=1}^k p_i < 1$}
                \State $\lambda_{\textrm{high}} = \lambda$
            \ELSE 
                \State $\lambda_{\textrm{low}} = \lambda$
            \ENDIF
        \ENDWHILE
        
        \State \textbf{return} $\lambda, p_1, p_2, \dots, p_k$.
	\end{algorithmic}
\end{algorithm}
\begin{remark}
An alternative method for finding optimal $\lambda$, and $p_i$'s is to sort $f_i$ values in $\mathcal{O}(k \log k)$ first, and then finding the smallest $f_i$ such that if we set $\lambda = f_i$, the sum of $p_i$'s is bigger than $1$ (let $j$ be the index of that value). Without loss of generality, assume that  $f_1 \leq \dots \leq f_k$. Then, $\sum_{i = j}^k \frac{-\lambda + f_i}{2\delta} = 1$, which has a closed-form solution with respect to $\lambda$.   
\end{remark} 
To update $\bw$, we need to solve the following optimization problem:
\begin{equation}
\label{eq: robust_logistic_regression_minimization}
\arraycolsep=1.4pt\def\arraystretch{1.3}
\begin{array}{lll}
    \displaystyle{\min_{\mathbf{w}}} \quad &  \displaystyle{ \sum_{i=1}^{k} p^*_i f_i(\bw) + \sum_{j=1}^{k} q^*_j g_i(\bw)}, \\
\end{array}
\end{equation}
Similar to the standard statistical learning framework, we solve the following empirical risk minimization problem by applying the gradient descent to $\bw$ on a finite data sample. Define $\hat{f}_i$ as follows:
\begin{equation}
\label{eq: f_i_erm_definition}
\hat{f}_i (\bw) = \pi_1 \sum_{t=1}^n \Big[- \log \Big(\sigma(\mathbf{w}^T \mathbf{x}_t) \Big) \Big],  
\end{equation}
where $\bx_1, \dots, \bx_n$ are generated from the distribution $\mathcal{N}(\bmu_1^*, \boldsymbol{\Sigma}_{1i})$.
The empirical risk minimization problem can be written as follows: 
\begin{equation}
\label{eq: erm_logistic_regression_minimization}
\arraycolsep=1.4pt\def\arraystretch{1.3}
\begin{array}{lll}
    \displaystyle{\min_{\mathbf{w}}} \quad &  \displaystyle{ \sum_{i=1}^{k} p^*_i \hat{f}_i(\bw) + \sum_{j=1}^{k} q^*_j \hat{g}_i(\bw)}, 
\end{array}
\end{equation}
Algorithm~\ref{alg: robust_normal_discriminant analysis} summarizes the robust linear discriminant analysis method for the case where the label of all data points is available.
\begin{algorithm}
	\caption{Robust Quadratic Discriminant Analysis in the Presence of Missing Values} 
	\label{alg: robust_normal_discriminant analysis}
	\begin{algorithmic}[1]
        \State \textbf{Input}: $\bX_0, \bX_1$: matrix of data points with labels $0$ and $1$ respectively, $T:$ Number of iterations, $\alpha:$ Step-size.   
        
        \State Estimate $\boldsymbol{\mu}_{\textrm{min}_0}$ and $\boldsymbol{\mu}_{\textrm{max}_0}$ using the available entries of $\bX_0$.
        
        \State Estimate $\boldsymbol{\mu}_{\textrm{min}_1}$ and $\boldsymbol{\mu}_{\textrm{max}_1}$ using the available entries of $\bX_1$.

        \State Estimate $\mathbf{\Sigma}_{01}, \dots, \mathbf{\Sigma}_{0k}$ using bootstrap estimator on the available data of $\bX_0$.
        
        \State Estimate $\boldsymbol{\Sigma}_{11}, \dots, \boldsymbol{\Sigma}_{1k}$ using bootstrap estimator on the available data of $\bX_1$.
        
        \FOR {$i = 1, \ldots, T$}
            \State Compute $\bmu^*_1$ and $\bmu^*_0$ by Equation~\eqref{eq: mu_update}.
            
            \State Find optimal $p_1, \dots, p_k$, and $q_1, \dots, q_k$ using Algorithm~\ref{alg: bisection}.
            
            \State $\bw = \bw - \alpha \Bigg( \sum_{i=1}^{k} p^*_i \nabla \hat{f}_i(\bw) + \sum_{j=1}^{k} q^*_j \nabla \hat{g}_i(\bw) \Bigg)$
        \ENDFOR
        
	\end{algorithmic}
\end{algorithm}
Theorem~\ref{thm: Robust_Logistic_Regression_Convergence} demonstrates the convergence of gradient descent algorithm applied to~\eqref{eq: erm_logistic_regression_minimization} in $\mathcal{O} \Big(\frac{k}{\epsilon} \log(\frac{M}{\epsilon})\Big)$ iterations to an $\epsilon$-optimal solution. 
\begin{theorem}\label{thm: Robust_Logistic_Regression_Convergence}
Assume that $M = \max_{i} f_i$. Gradient descent algorithm requires  $\mathcal{O} \Big(\frac{k}{\epsilon} \log(\frac{M}{\epsilon})\Big)$ gradient evaluations for converging to an $\epsilon$-optimal saddle point of the optimization problem~\eqref{eq: erm_logistic_regression_minimization}.
\end{theorem}
 
In Appendix~\ref{appendix: robust_classification}, we extend the methodology to the case where $y$ contains missing entries.
\section{Experiments}
\label{sec: numerical_results}
In this section, we evaluate RIFLE's performance on a diverse set of inference tasks in the presence of missing values. We compare RIFLE's performance to several state-of-the-art approaches for data imputation on synthetic and real-world datasets. 
The experiments are designed in a manner that the sensitivity of the model to factors such as the number of samples, data dimension, types, and proportion of missing values can be evaluated. The description of all datasets used in the experiments can be found in Appendix~\ref{appendix: datasets}.
\subsection{Evaluation Metrics}
We need access to the ground-truth values of the missing entries to evaluate RIFLE and other state-of-the-art imputation approaches. Hence, we artificially mask a proportion of available data entries and predict them with different imputation methods. A method performs better than others if the predicted missing entries are closer to the ground-truth values. To measure the performance of RIFLE and the existing approaches on a regression task for a given test dataset consisting of $N$ data points, we use normalized root mean squared error (NRMSE), defined as:
\begin{equation*}
    \textrm{NRMSE} = \frac{\sqrt{\frac{1}{N} \sum_{i=1}^{N} (y_i - \hat{y}_i)^2}}{\sqrt{\frac{1}{N} \sum_{i=1}^{N} (y_i - \bar{y})^2}}  
\end{equation*}
where $y_i$, $\hat{y}_i$, and $\bar{y}$ represent the true value of the $i$-th data point, the predicted value of the $i$-th data point, and the average of true values of data points, respectively. In all experiments, generated missing entries follow either a missing completely at random (MCAR) or a missing not at random (MNAR) pattern. A discussion on the procedure of generating these patterns can be found in Appendix~\ref{appendix: MNAR}.
\subsection{Tuning Hyper-parameters of RIFLE}
The hyper-parameter $c$ in~\eqref{robust_linear_regression2} controls the robustness of the model by adjusting the size of confidence intervals. This parameter is tuned by performing a cross-validation procedure over the set $\{0.1, 0.25, 0.5, 1, 2, 5, 10, 20,$ $50, 100\}$, and the one with the lowest NMRSE is chosen. The default value in the implementation is $c=1$ since it consistently performs well over different experiments. Furthermore, $\lambda$, the hyper-parameter for the ridge regression regularizer, is tuned by choosing $20\%$ of the data as the validation set from the set $\{0.01, 0.1, 0.5, 1, 2, 5, 10, 20, 50\}$. To tune $K$, the number of bootstrap samples for estimating the confidence intervals, we tried $10, 20, 50$, and $100$. No significant difference is observed in terms of the test performance for the above values. 

Furthermore, we tune the hyper-parameters of the competing packages as follows. For KNN-Imputer~\citep{troyanskaya2001missing}, we try $\{2, 10, 20, 50\}$ for the number of neighbors ($K$) and pick the one with the highest performance. For MICE~\citep{buuren2010mice} and Amelia~\citep{honaker2011amelia}, we generate $5$ different imputed data and pick the one with the highest performance on the test data. MissForest has multiple hyper-parameters. We keep the criterion as ``MSE'' since our performance evaluation measure is NRMSE. Moreover, we tune the number of iterations and number of estimations (number of trees) by checking values from $\{5, 10, 20\}$ and $\{50, 100, 200\}$, respectively. We do not change the structure of the neural networks for MIDA~\citep{gondara2018mida} and GAIN~\citep{yoon2018gain}, and the default versions are performed for imputing datasets.

\vspace{-2mm}
\subsection{RIFLE Consistency}
\vspace{-2mm}
In Theroem~\ref{thm: regression_performance} Part (a), we demonstrated that RIFLE is consistent.  
\begin{figure}[ht]
\centering
\includegraphics[width=0.8\columnwidth]{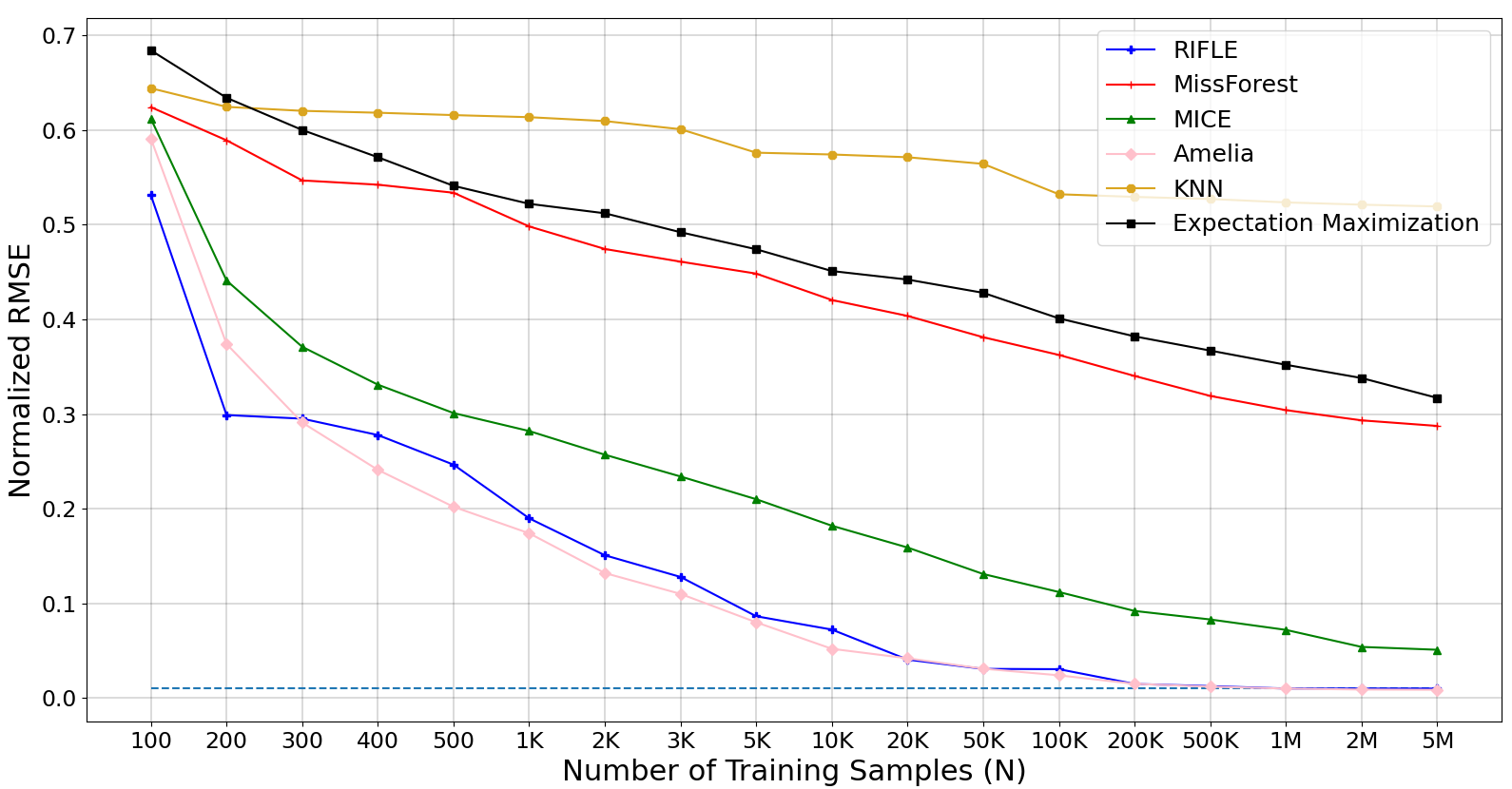}
\caption{Comparing the consistency of RIFLE, MissForest, KNN Imputer, MICE, Amelia, and Expectation Maximization methods on a synthetic dataset containing $40\%$ of missing values.}
\label{fig: consistency_different_methods}
\end{figure}
In Figure~\ref{fig: consistency_different_methods}, we investigate the consistency of RIFLE on synthetic datasets with different proportions of missing values. The synthetic data has $50$ input features following a jointly normal distribution with the mean whose entries are randomly chosen from the interval $(-100, 100)$. Moreover, the covariance matrix equals $\Sigma = S S^T$ where $S$ elements are randomly picked from $(-1, 1)$. The dimension of $S$ is $50 \times 20$.  The target variable is a linear function of input features added to a mean zero normal noise with a standard deviation of $0.01$. As depicted in Figure~\ref{fig: consistency_different_methods}, RIFLE requires fewer samples to recover the ground-truth parameters of the model compared to MissForest, KNN Imputer, Expectation Maximization~\citep{dempster1977maximum}, and MICE. Amelia's performance is significantly good since the predictors have a joint normal distribution and the linear underlying model. Note that by increasing the number of samples, the NRMSE of our framework converges to $0.01$, which is the standard deviation of the zero-mean Gaussian noise added to each target value (the dashed line).

\begin{table}[]
\centering
\vspace{-5mm}
\resizebox{\textwidth}{!}{
\begin{tabular}{|c|c|c|c|c|c|c|c|c|}
\hline
\textbf{Dataset Name} & \textbf{RIFLE} & \textbf{QRIFLE} & \textbf{MICE} & \textbf{Amelia} & \textbf{GAIN} & \textbf{MissForest} & \textbf{MIDA} & \textbf{EM} \\ \hline
Spam (30\%)           & {\ul 0.87} $\pm 0.009$  & \textbf{0.82} $\pm 0.009$ & 1.23 $\pm 0.012$          & 1.26 $\pm 0.007$            & 0.91 $\pm 0.005$          & 0.90 $\pm 0.013$ & 0.97 $\pm 0.008$ &  $0.94 \pm 0.004$         \\ \hline
Spam (50\%)           & {\ul 0.90} $\pm 0.013$ & \textbf{0.86} $\pm 0.014$ & 1.29 $\pm 0.018$         & 1.33 $\pm 0.024$            & 0.93 $\pm 0.015$         & 0.92 $\pm 0.011$          & 0.99 $\pm 0.011$ & $0.97 \pm 0.008$    \\ \hline
Spam (70\%)           & {\ul 0.92} $\pm 0.017$ & \textbf{0.91} $\pm 0.019$ & 1.32 $\pm 0.028$          & 1.37 $\pm 0.032$   & 0.97 $\pm 0.014$ & 0.95 $\pm 0.016$  & 0.99 $\pm 0.018$ & $0.98 \pm 0.017$ \\ \hline
Housing (30\%)        & 0.86 $\pm 0.015$     &  0.89 $\pm 0.018$     & 1.03 $\pm 0.024$          & 1.02 $\pm 0.016$            & \textbf{0.82} $\pm 0.015$ & {\ul 0.84} $\pm 0.018$          & 0.93 $\pm 0.025$ & $0.95 \pm 0.011$         \\ \hline
Housing (50\%)        & \textbf{0.88} $\pm 0.021$ & 0.90 $\pm 0.024$ & 1.14 $\pm 0.029$          & 1.09 $\pm 0.027$            & \textbf{0.88} $\pm 0.019$ & \textbf{0.88} $\pm 0.018$  & 0.98 $\pm 0.029$ & $0.96 \pm 0.016$        \\ \hline
Housing (70\%)        & \textbf{0.92} $\pm 0.026$  & 0.95 $\pm 0.028$ & 1.22 $\pm 0.036$          & 1.18 $\pm 0.038$            & 0.95 $\pm 0.027$         & {\ul 0.93} $\pm 0.024$          & 1.02 $\pm 0.037$  & $0.98 \pm 0.017$        \\ \hline
Clouds (30\%)         & 0.81 $\pm 0.018$ & 0.79 $\pm 0.019$           & 0.98 $\pm 0.024$          & 1.04 $\pm 0.027$            & {\ul 0.76} $\pm 0.021$    & \textbf{0.71} $\pm 0.011$       & 0.83 $\pm 0.022$ & $0.86 \pm 0.013$ \\ \hline
Clouds (50\%)         & 0.84 $\pm 0.026$ & 0.84 $\pm 0.028$  & 1.10 $\pm 0.041$          & 1.13 $\pm 0.046$            & {\ul 0.82} $\pm 0.027$    & \textbf{0.75} $\pm 0.023$       & 0.88 $\pm 0.033$  & $0.89 \pm 0.018$ \\ \hline
Clouds (70\%)         & {\ul 0.87} $\pm 0.029$ & 0.90 $\pm 0.033$    & 1.16 $\pm 0.044$         & 1.19 $\pm 0.048$            & 0.89 $\pm 0.035$          & \textbf{0.81} $\pm 0.031$       & 0.93 $\pm 0.044$  & $0.92 \pm 0.023$        \\ \hline
Breast Cancer (30\%)  & \textbf{0.52} $\pm 0.023$ & {\ul 0.54} $\pm 0.027$ & 0.74 $\pm 0.031$          & 0.81 $\pm 0.032$  & 0.58 $\pm 0.024$          & 0.55 $\pm 0.016$ & 0.70 $\pm 0.026$ & $0.67 \pm 0.014$ \\ \hline
Breast Cancer (50\%)  & \textbf{0.56} $\pm 0.026$ & 0.59 $\pm 0.027$ & 0.79 $\pm 0.029$          & 0.85 $\pm 0.033$            & 0.64 $\pm 0.025$          & {\ul 0.59} $\pm 0.022$          & 0.76 $\pm 0.035$ & $0.69 \pm 0.022$ \\ \hline
Breast Cancer (70\%)  & \textbf{0.59} $\pm 0.031$ & 0.65 $\pm 0.034$ & 0.86 $\pm 0.042$          & 0.92 $\pm 0.044$            & 0.70 $\pm 0.037$          & {\ul 0.63} $\pm 0.028$          & 0.82 $\pm 0.035$ & $0.67 \pm 0.014$ \\ \hline
Parkinson (30\%)      & 0.57 $\pm 0.016$  & 0.55 $\pm 0.016$         & 0.71 $\pm 0.019$         & 0.67 $\pm 0.021$            & \textbf{0.53} $\pm 0.015$ & {\ul 0.54} $\pm 0.010$          & 0.62 $\pm 0.017$  & $0.64 \pm 0.011$  \\ \hline
Parkinson (50\%)      & {\ul 0.62} $\pm 0.022$ & 0.64 $\pm 0.025$    & 0.77 $\pm 0.029$          & 0.74 $\pm 0.034$            & \textbf{0.61} $\pm 0.022$ & 0.65 $\pm 0.014$                & 0.71 $\pm 0.027$   & $0.69 \pm 0.022$       \\ \hline
Parkinson (70\%)      & \textbf{0.67} $\pm 0.027$ & 0.74 $\pm 0.033$ & 0.85 $\pm 0.038$          & 0.82 $\pm 0.037$            & {\ul 0.69} $\pm 0.031$    & 0.73 $\pm 0.022$                & 0.78 $\pm 0.038$     & $0.75 \pm 0.029$    \\ \hline
\end{tabular}
}
\vspace{-3mm}
\caption{Performance comparison of RIFLE, QRIFLE (Quadratic RIFLE), and state-of-the-art methods on several UCI datasets. We applied to impute methods on three different missing-value proportions for each dataset. The best imputer is highlighted with bold font, and the second-best imputer is underlined. Each experiment is done $5$ times, and the average and the standard deviation of performances are reported.}
\label{tab:imputation_uci}
\end{table}

\vspace{-3mm}
\subsection{Data Imputation via RIFLE}
\vspace{-2mm}
As explained in Section~\ref{sec: robust_linear_regression}, while the primary goal of RIFLE is to learn a robust regression model in the presence of missing values, it can also be used as an imputation tool. We run RIFLE and several state-of-the-art approaches on five datasets from the UCI repository~\citep{Dua:2019} (Spam, Housing, Clouds, Breast Cancer, and Parkinson datasets) with different proportions of MCAR missing values (the description of the datasets can be found in Appendix~\ref{appendix: datasets}). Then, we compute the NMRSE of imputed entries. Table~\ref{tab:imputation_uci} shows the performance of RIFLE compared to other approaches for the datasets where the proportion of missing values are relatively high $\Big(\frac{n(1-p)}{d} \approx  \mathcal{O}(1)\Big)$. RIFLE outperforms these methods in almost all cases and performs slightly better than MissForest, which uses a highly non-linear model (random forest) to impute missing values.

\subsection{Sensitivity of RIFLE to the Number of Samples and Proportion of Missing Values}
In this section, we analyze the sensitivity of RIFLE and other state-of-the-art approaches to the number of samples and the proportion of missing values. 
\begin{figure}[ht]
\centering
\vspace{-3mm}
\includegraphics[width=0.8\columnwidth]{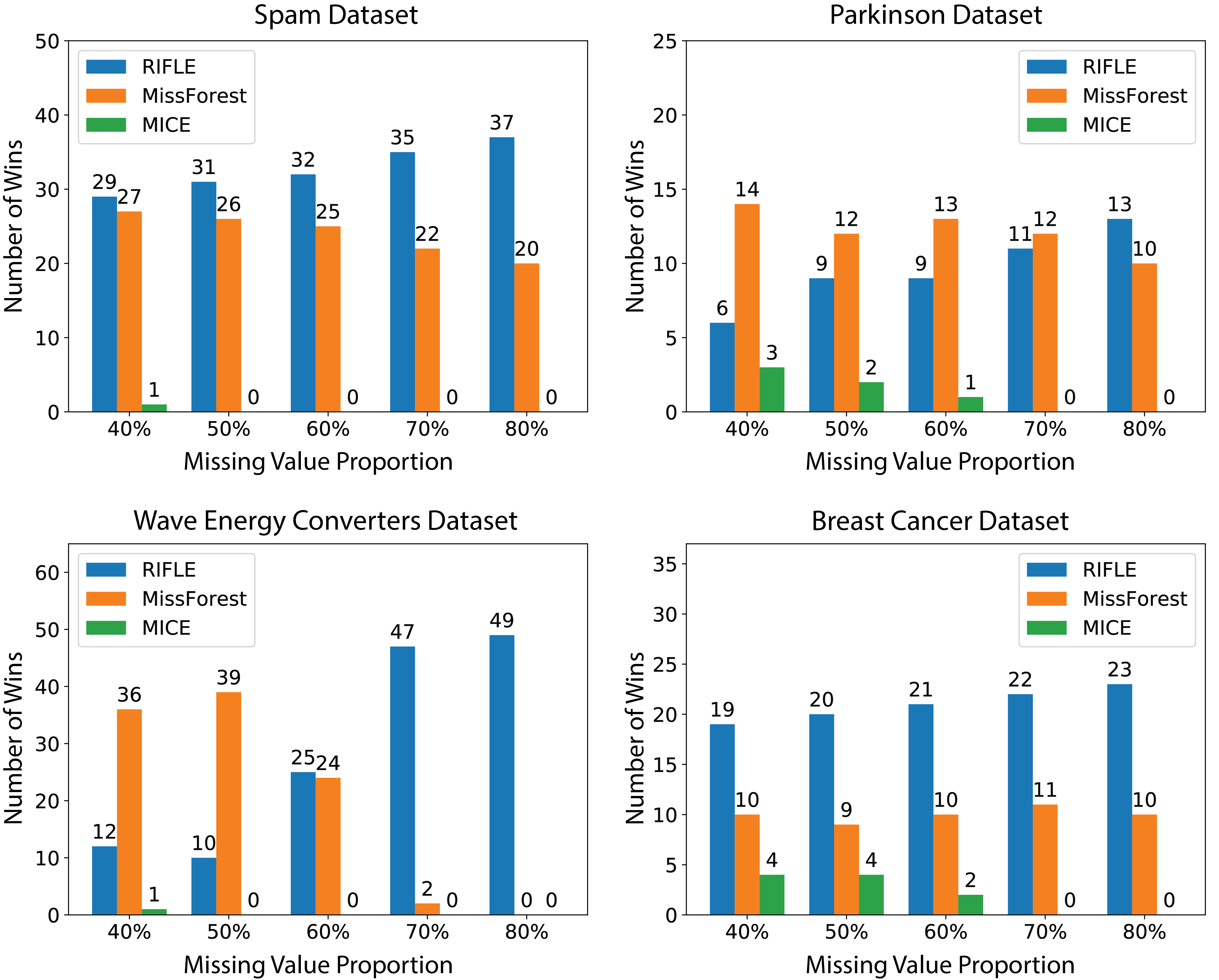}
\vspace{-3mm}
\caption{Performance Comparison of RIFLE, MICE, and MissForest on four UCI datasets: Parkinson, Spam, Wave Energy Converter, and Breast Cancer. For each dataset, we count the number of features that each method outperforms the others.}
\label{fig: counts}
\end{figure}
In the experiment in Figure~\ref{fig: counts}, we create $5$ datasets containing $40\%$, $50\%$, $60\%$, $70\%$, and $80\%$ of MCAR missing values, respectively, for four real datasets (Spam, Parkinson, Wave Energy Converter, and Breast Cancer) from UCI Repository~\citep{Dua:2019} (the description of the datasets can be found in Appendix~\ref{appendix: datasets}). Given a feature in a dataset containing missing values, we say an imputer wins that feature if the imputation error in terms of NRMSE for that imputer is less than the error of the other imputers. Figure~\ref{fig: counts} reports the number of features won by each imputer on the created datasets described above. As we observe, the number of wins for RIFLE increases as we increase the proportion of missing values. This observation shows that the sensitivity of RIFLE as an imputer to the proportion of missing values is less than MissForest and MICE in general.

\begin{figure}[ht]
\centering
\vspace{-3mm}
\includegraphics[width=0.8\columnwidth]{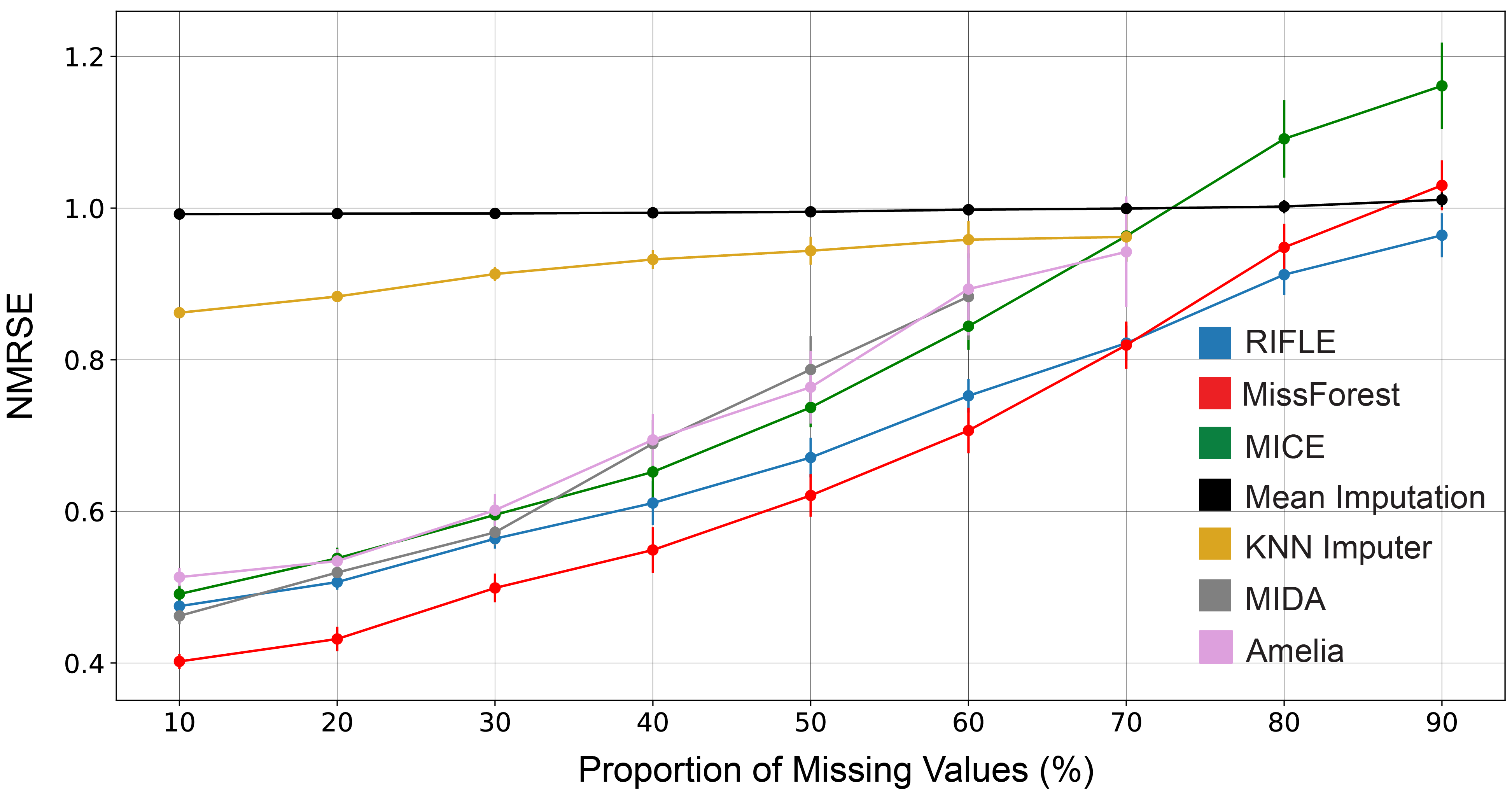}
\vspace{-3mm}
\caption{Sensitivity of RIFLE, MissForest, Amelia, KNN Imputer, MIDA, and Mean Imputer to the percentage of missing values on the Drive dataset. Increasing the percentage of missing value entries degrades the benchmarks' performance compared to RIFLE. KNN-imputer implementation cannot be executed on datasets containing  $80\%$ (or more) missing entries. Moreover, Amelia and MIDA do not converge to a solution when the percentage of missing value entries is higher than $70\%$.}
\label{fig: sensitivity_missing_value_proportion}
\end{figure}
Figure~\ref{fig: counts} does not show how the increase in the proportion of missing values changes the NRMSE of imputers. Next, we analyze the sensitivity of RIFLE and several imputers to change in missing value proportions. Fixing the proportion of missing values, we generate $10$ random datasets containing missing values in random locations on the Drive dataset (the description of datasets is available in Appendix~\ref{appendix: datasets}). We impute the missing values for each dataset with RIFLE, MissForest, Mean Imputation, and MICE. Figure~\ref{fig: sensitivity_missing_value_proportion} shows the average and the standard deviation of these $4$ imputers' performances for different proportions of missing values ($10\%$ to $90\%$). 
Figure~\ref{fig: sensitivity_missing_value_proportion} depicts the sensitivity of MissForest and RIFLE to the proportion of missing values in the Drive dataset. We select 400 data points for each experiment with different proportions of missing values (from $10\%$ to $90\%$) and report the average NRMSE of imputed entries.  Finally, in Figure~\ref{fig: sensitivity_number_of_samples}, we have evaluated RIFLE and other methods on the BlogFeedback dataset (see Appendix~\ref{appendix: datasets}) containing $40\%$ missing values. The results show that RIFLE's performance is less sensitive to decreasing the number of samples.
\begin{figure}[ht]
\centering
\includegraphics[width=0.8\columnwidth]{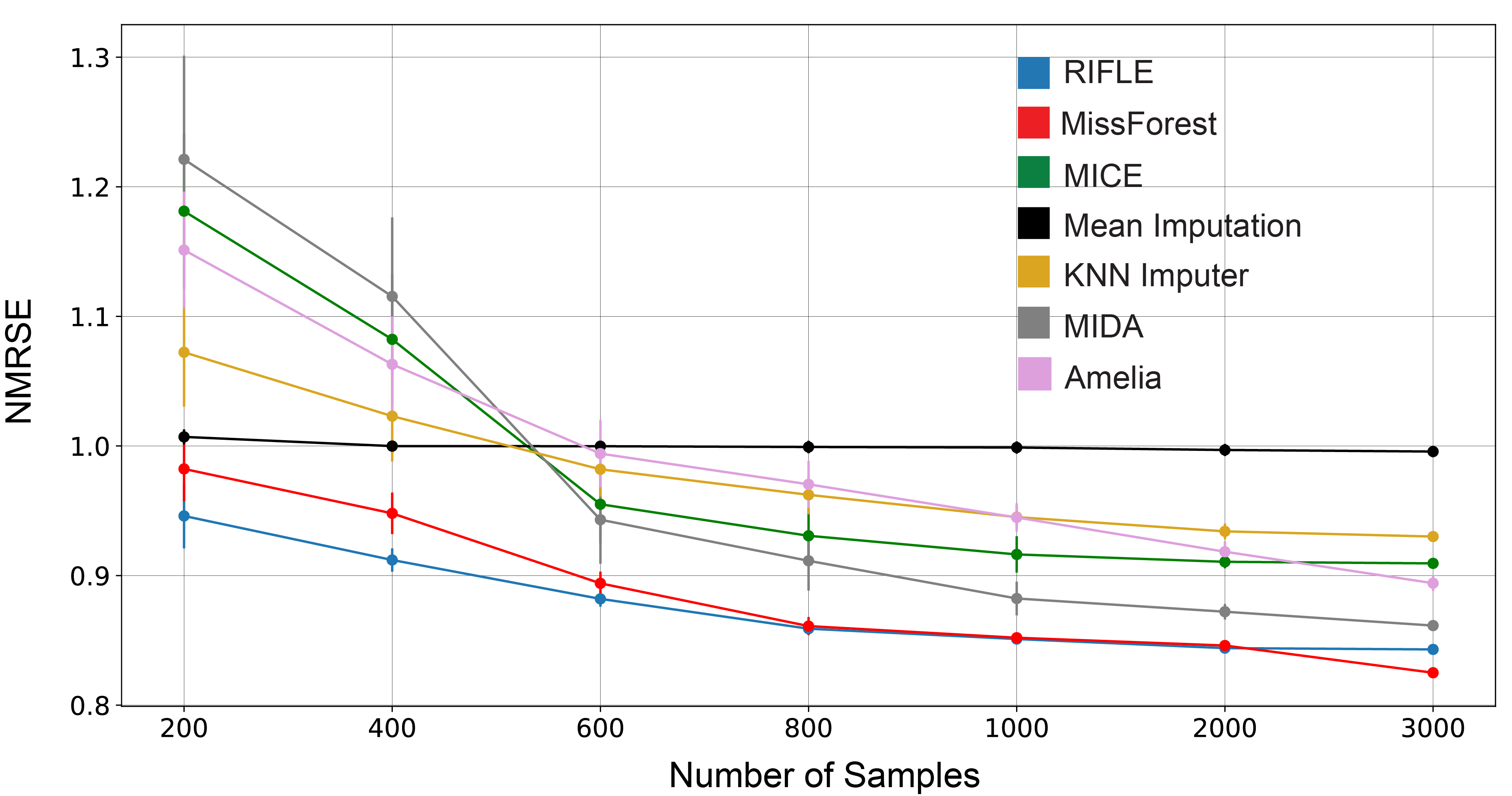}
\vspace{-3mm}
\caption{Sensitivity of RIFLE, MissForest, MICE, Amelia, Mean Imputer, KNN Imputer, and MIDA to the number of samples for the imputations of Blog Feedback dataset containing $40\%$ of MCAR missing values. When the number of samples is limited, RIFLE outperforms other methods, and its performance is very close to the non-linear imputer MissForest for larger samples.}
\label{fig: sensitivity_number_of_samples}
\end{figure}

\subsection{Performance Comparison on Real Datasets}
In this section, we compare the performance of RIFLE to several state-of-the-art approaches, including MICE~\citep{buuren2010mice}, Amelia~\citep{honaker2011amelia}, MissForest~\citep{stekhoven2012missforest}, KNN Imputer~\citep{raghunathan2001multivariate}, and MIDA~\citep{gondara2018mida}. There are two primary ways to do this. One method to predict a continuous target variable in a dataset with many missing values is first to impute the missing data with a state-of-the-art package, then run a linear regression. An alternative approach is to directly learn the target variable, as we discussed in Section~\ref{sec: robust_linear_regression}.

Table~\ref{tab:real_datasets_mnar} compares the performance of mean imputation, MICE, MIDA, MissForest, and KNN to that of RIFLE on three datasets: NHANES, Blog Feedback, and superconductivity. Both Blog Feedback and Superconductivity datasets contain $30\%$ of MNAR missing values generated by Algorithm~\ref{alg: MNAR_Generator}, with $10000$ and $20000$ training samples, respectively. The description of the NHANES data and its distribution of missing values can be found in Appendix~\ref{appendix: datasets}. 

\noindent{\textbf{Efficiency of RIFLE:}} We perform RIFLE for $1000$ iterations and the step size of $0.01$ in the above experiments. At each iteration, the main operation is to find the optimal $\btheta$ for any given $\bb$ and $\bC$. The average time of each method on each dataset is reported in Table~\ref{tab:time} in Appendix~\ref{appendix: time}. The main reason for the time efficiency of RIFLE compared to MICE, MissForest, MIDA, and KNN Imputer is that it directly predicts the target variable without imputation of all missing entries.

\begin{table}[ht]
\begin{center}
\begin{tabular}{|c|c|c|c|c}
\hline
\multirow{2}{*}{\textbf{Methods}} & \multicolumn{3}{c|}{\textbf{Datasets}}      \\ \cline{2-4} 
                                  & Super Conductivity & Blog Feedback & NHANES \\ \hline
Regression on Complete Data       & 0.4601             & 0.7432        & 0.6287 \\ \hline
\textbf{RIFLE}                             & $\textbf{0.4873} \pm 0.0036$             & $0.8326 \pm 0.0085$        & $\textbf{0.6304} \pm 0.0027$ \\ \hline
Mean Imputer + Regression      & $0.6114 \pm 0.0006$             & $0.9235 \pm 0.0003$        & $0.6329 \pm 0.0008$ \\ \hline
MICE + Regression      & $0.5078 \pm {0.0124}$             & $0.8507 \pm 0.0325$        & $0.6612 \pm 0.0282$ \\ \hline
EM + Regression                       & $0.5172 \pm 0.0162$              & $0.8631 \pm 0.0117$        & $0.6392 \pm 0.0122$ \\ \hline
MIDA Imputer + Regression     & $0.5213 \pm 0.0274$             & $0.8394 \pm 0.0342$        & $0.6542 \pm 0.0164$ \\ \hline
MissForest                        & $0.4925 \pm 0.0073$             & $\textbf{0.8191} \pm 0.0083$        & $0.6365 \pm 0.0094$  \\ \hline
KNN Imputer                       & $0.5438 \pm 0.0193$              & $0.8828 \pm 0.0124$        & $0.6427 \pm 0.0135$ \\ \hline
\end{tabular}    
\end{center}
\caption{Normalized RMSE of RIFLE and several state-of-the-art Methods on Superconductivity, blog feedback, and NHANES datasets. The first two datasets contain $30\%$ Missing Not At Random (MNAR) missing values in the training phase generated by Algorithm~\ref{alg: MNAR_Generator}. Each method applied $5$ times to each dataset, and the result is reported as the average performance $\pm$ standard deviation of experiments in terms of NRMSE.}
\label{tab:real_datasets_mnar} 
\end{table}
Since MICE and MIDA cannot predict values during the test phase without data imputation, we use them in a pre-processing stage to impute the data. Then we apply the linear regression to the imputed dataset. On the other hand, RIFLE, KNN imputer, and MissForest can predict the target variable without imputing the training dataset. Table~\ref{tab:real_datasets_mnar} shows that RIFLE outperforms all other state-of-the-art approaches executed on the three mentioned datasets. In particular, RIFLE outperforms MissForest, while the underlying model RIFLE uses is simpler (linear) compared to the nonlinear random forest model utilized by Missforest.

\subsubsection{Performance of RIFLE on Classification Tasks}
In Section~\ref{sec: robust_classification}, we discussed how to specialize RIFLE to robust normal discriminant analysis in the presence of missing values. Since the maximization problem over the second moments of the data ($\boldsymbol{\Sigma}$) is intractable, we solved the maximization problem over a set of $k$ covariance matrices estimated by bootstrap sampling. To investigate the effect of choosing $k$ on the performance of the robust classifier, we train robust normal discriminant analysis models for different values of $k$ on two training datasets (Avila and Magic) containing $40\%$ MCAR missing values. The description of the datasets can be found in Appendix~\ref{appendix: datasets}. For $k=1$, there is no maximization problem, and thus, it is equivalent to the classifier proposed in~\citet{fung1989treatment}. As shown in Figure~\ref{fig: number_of_estimations}, increasing the number of covariance estimations generally enhances the accuracy of the classifier in the test phase. However, as shown in Theorem~\ref{thm: Robust_Logistic_Regression_Convergence}, the required time for completing the training phase grows linearly regarding the number of covariance estimations.   
\begin{figure}[ht]
\centering
\vspace{-4mm}
\includegraphics[width=1\columnwidth]{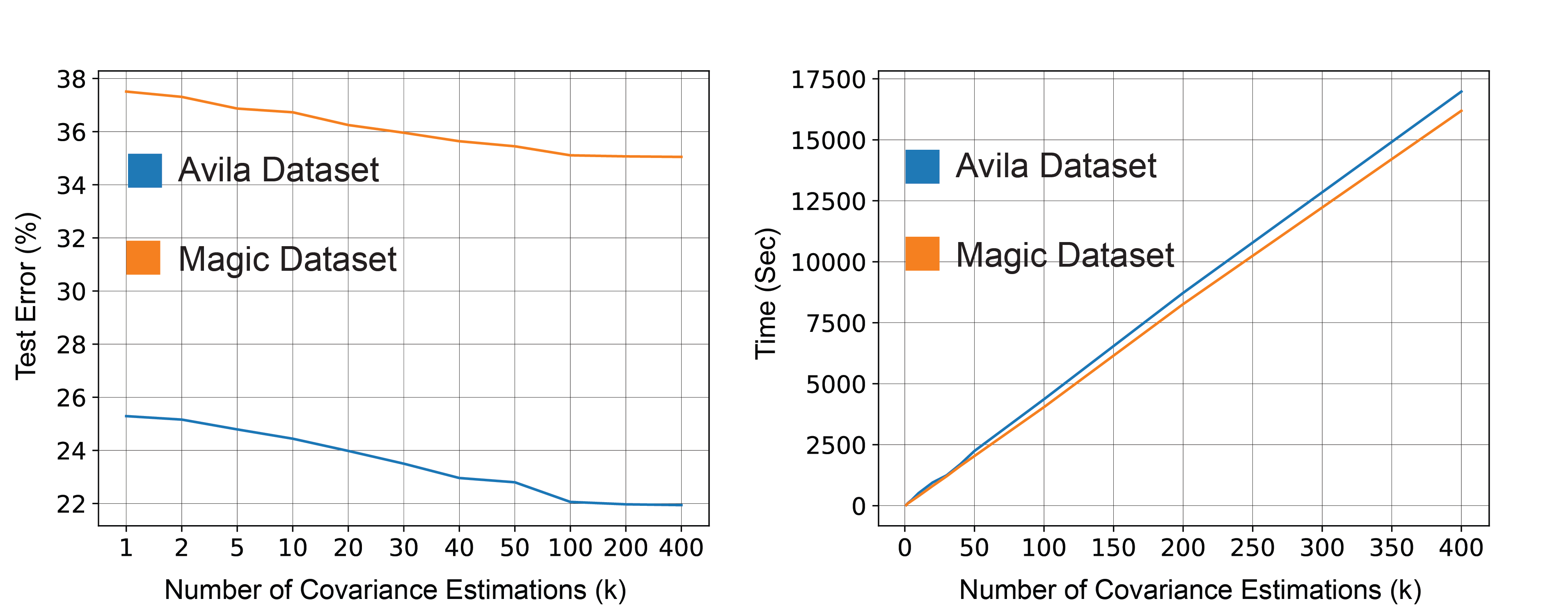}
\caption{Effect of the number of covariance estimations on the performance (left) and run time (right) of robust LDA on Avila and Magic datasets. Increasing the number of covariance estimations ($k$) improves the model's accuracy on the test data. However, it takes longer training time. }
\label{fig: number_of_estimations}
\end{figure}

\begin{table}
\resizebox{\textwidth}{!}{
\begin{tabular}{|c|c|c|c|}
\hline
\multirow{2}{*}{\textbf{Number of Training Data Points}} & \multicolumn{3}{c|}{\textbf{Method}}                  \\ \cline{2-4} 
                                                         & LDA     & Robust LDA & Robust QDA \\ \hline
50                                                       & $52.38\% \pm 3.91\%$ & $62.14\% \pm 1.78\%$  & $61.36\% \pm 1.62\%$    \\ \hline
100                                                      & $61.24\% \pm 1.89\%$ & $68.46\% \pm 1.04\%$  & $70.07\% \pm 0.95\%$    \\ \hline
200                                                      & $73.49\% \pm 0.97\%$ & $73.35\% \pm 0.67\%$                        & $73.51\% \pm 0.52\%$    \\ \hline
\end{tabular}
}
\caption{Sensitivity of Linear Discriminant Analysis, Robust LDA (Common Covariance Matrices), and Robust QDA (Different Covariance matrices for two groups) to the number of training samples.}
\label{tab: lda_vs_robust}
\end{table}
\vspace{-3mm}
\subsubsection{Comparison of Robust Linear Regression and Robust QDA}
\vspace{-2mm}
An alternative approach to the robust QDA presented in Section~\ref{sec: robust_classification} is to apply the robust linear regression algorithm (Section~\ref{sec: robust_linear_regression}) and mapping the solutions to each one of the classes by thresholding (positive value maps to Label $1$ and negative values to label $-1$). 

Table~\ref{tab:robust_linear_regression_vs_lda} compares the performance of two classifiers on three different datasets. As demonstrated in the table, when all features are continuous, quadratic discriminant analysis has a better performance. It shows the QDA model relies highly on the normality assumption, while robust linear regression handles the categorical features better than robust QDA.
\begin{table}
\centering
\resizebox{\textwidth}{!}{
\begin{tabular}{cc|c|c|c|c|c|c|c}
\cline{3-8}
\multicolumn{1}{l}{}                       & \multicolumn{1}{l|}{} & \multicolumn{6}{c|}{\textbf{Accuracy of Methods}}                      \\ \hline
\multicolumn{1}{|c|}{\textbf{Dataset}}              & Feature Type          & RIFLE   & Robust QDA & MissForest & MICE & KNN Imputer & EM \\ \hline
\multicolumn{1}{|c|}{Glass Identification} & Continuous            & $67.12\% \pm 1.84\%$ & $\mathbf{69.54\% \pm 1.97\%}$  & $65.76\% \pm 1.49\%$     & $62.48\% \pm 2.45\%$ & $60.37\% + \pm 1.12\%$ & $68.21\% + \pm 0.94\%$     \\ \hline
\multicolumn{1}{|c|}{Annealing}            & Mixed                 & $63.41\% \pm 2.44\%$ & $59.51\% \pm 2.21\%$    & $\mathbf{64.91\% \pm 1.35\%}$    & $60.66\% \pm 1.59\%$ & $57.44\% \pm 1.44\%$ & $59.43\% + \pm 1.29\%$ \\ \hline
\multicolumn{1}{|c|}{Abalone}              & Mixed                 & $68.41\% \pm 0.74\%$ & $63.27\% \pm 0.76\%$    & $\mathbf{69.40\% \pm 0.42\%}$    & $63.12\% \pm 0.98\%$ & $62.43\% \pm 0.38\%$ & $62.91\% + \pm 0.37\%$    \\ \hline
\multicolumn{1}{|c|}{Lymphography}         & Discrete              & $\mathbf{66.32\% \pm 1.05\%}$ & $58.15\% \pm 1.21\%$    & $66.11\% \pm 0.94\%$    & $55.73\% \pm 1.24$ & $57.39\% \pm 0.88\%$ & $59.55\% + \pm 0.68\%$    \\ \hline
\multicolumn{1}{|c|}{Adult}                & Discrete              & $\mathbf{72.42\% \pm 0.06\%}$ & $60.36\% \pm 0.08$    & $70.34\% \pm 0.03\%$    & $63.30\% \pm 0.14\%$ & $60.14\% \pm 0.00$ & $60.69\% + \pm 0.01\%$     \\ \hline
\end{tabular}
}
\caption{Accuracy of RIFLE, MICE, KNN-Imputer, Expectation Maximization (EM), and Robust QDA on different discrete, mixed, and continuous datasets. Robust QDA can perform better than other methods when the input features are continuous, and the target variable is discrete. However, RIFLE results in higher accuracy in mixed and discrete settings.}
\label{tab:robust_linear_regression_vs_lda}
\end{table}


\noindent \textbf{Limitations and Future Directions:} The proposed framework for robust regression in the presence of missing values is limited to linear models. While in Appendix~\ref{appendix: qrifle}, we use polynomial kernels to apply non-linear transformations on the data, such an approach can potentially increase the number of missing values in the kernel space generated by the composition of the original features. A future direction is to develop efficient algorithms for non-linear regression models such as multi-layer neural networks, decision tree regressors, gradient boosting regressors, and support vector regression models. In the case of robust classification, the methodology is extendable to any loss beyond quadratic discriminant analysis. Unlike the regression case, a limitation of the proposed method for robust classification is its reliance on the Gaussianity assumption of data distribution (conditioned on each data label). A natural extension is to assume the underlying data distribution follows a mixture of Gaussian distributions. 

\noindent \textbf{Conclusion:} In this paper, we proposed a distributionally robust optimization framework over the distributions with the low-order marginals within the estimated confidence intervals for inference and imputation of datasets in the presence of missing values. We developed algorithms for regression and classification with convergence guarantees. The method's performance is evaluated on synthetic and real datasets with different numbers of samples, dimensions, missing value proportions, and types of missing values. In most experiments, RIFLE consistently outperforms other existing methods.

\acks{This work was supported by the NIH/NSF Grant 1R01LM013315-01, the NSF CAREER Award CCF-2144985, and the AFOSR Young Investigator Program Award FA9550-22-1-0192.} 

\newpage


\bibliography{references}
\bibliographystyle{tmlr}

\newpage
\appendix
\section{A Review of Missing Value Imputation Methods in the Literature}
\label{appendix: related_works}
The fundamental idea behind many data imputation approaches is that the missing values can be predicted based on the available data of other data points and correlated features. One of the most straightforward imputation techniques is to replace missing values by the mean (or median) of that feature calculated from what data is available see~\citet[Chapter~3]{little2019statistical}. However, this na\"ive approach ignores the correlation between features and does not preserve the variance of features. Another class of imputers has been developed based on the least-square methods \citep{raghunathan2001multivariate, kim2005missing, zhang2008sequential, cai2006iterated}. \citet{raghunathan2001multivariate} learns a linear model with multivariate Gaussian noise for the feature with the least missing entries. It repeats the same procedure on the updated data to impute the next feature with the least missing entries until all features are completely imputed. One drawback of this approach is that the error from the imputation of previous features can be propagated to subsequent features. To impute entries of a given feature in a dataset, \citet{kim2005missing} learns several univariate regression models that consider that feature as the response. Then it takes the average of these predictions as the final value of imputation. This approach fails to learn the correlations involving more than two features. 

Many more complex algorithms have been developed for imputation, although many are sensitive to initial assumptions and may not converge. For instance, KNN-Imputer imputes a missing feature of a data point by taking the mean value of the $K$ closest complete data points \citep{troyanskaya2001missing}. MissForest, on the other hand, imputes the missing values of each feature by learning a random forest classifier using other training data features \citep{stekhoven2012missforest}. MissForest does not need to assume that all features are continuous~\citep{honaker2011amelia} or categorical \citep{schafer1997analysis}. However, both KNN-imputer and MissForest do not guarantee statistical or computational convergence for their algorithms. Moreover, when the proportion of missing values is high, both are likely to have a severe drop in performance, as demonstrated in Section~\ref{sec: numerical_results}. The Expectation Maximization (EM) algorithm is another popular approach that learns the parameters of a prior distribution on the data using available values based on the EM algorithm of~\citet{dempster1977maximum}; see also \citet{ghahramani1994supervised} and \citet{honaker2011amelia}. The EM algorithm is also used in Amelia, which fits a jointly normal distribution to the data using EM and the bootstrap technique \citep{honaker2011amelia}. While Amelia demonstrates a superior performance on datasets following a normal distribution, it is highly sensitive to the violation of the normality assumption (as discussed in~\citet{bertsimas2017predictive}). \citet{ghahramani1994supervised} adopt the EM algorithm to learn a joint Bernoulli distribution for the categorical data and a joint Gaussian distribution for the continuous variables independently. While those algorithms can be viewed as inference methods based on low-order estimates of moments, they do not consider \emph{uncertainty} in such low-order moments estimates. By contrast, our framework utilizes \textit{robust optimization} to consider the uncertainty around the estimated moments. Moreover, our optimization procedure for imputation and prediction is guaranteed to converge despite some of the algorithms mentioned above. 

Another popular method for data imputation is multiple imputations by chained equations (MICE). MICE learns a parametric distribution for each feature conditional on the remaining features. For instance, it assumes that the current target variable is a linear function of other features with a zero-mean Gaussian noise. Each feature can have its distinct distribution and parameters (e.g., Poisson regression, logistic regression). Based on the learned parameters of conditional distributions, MICE can generate one or more imputed datasets~\citep{buuren2010mice}.
More recently, several neural network-based imputers have been proposed. GAIN (Generative Adversarial Imputation Network) learns a generative adversarial network based on the available data and then imputes the missing values using the trained generator~\citep{yoon2018gain}. One advantage of GAIN over other existing GAN imputers is that it does not need a complete dataset during the training phase. MIDA (Multiple Imputation using Denoising Autoencoders) is an auto-encoder-based approach that trains a denoising auto-encoder on the available data considering the missing entries as noise. Similar to other neural network-based methods, these algorithms suffer from their black-box nature. They are challenging to interpret/explain, making them unpopular in mission-critical healthcare approaches. In addition, no statistical or computational guarantees are provided for these algorithms.

\citet{bertsimas2017predictive} formulates the imputation task as a constrained optimization problem where the constraints are determined by the underlying classification model such as KNN ($k$-nearest neighbors), SVM (Support Vector Machine), and Decision Trees. Their general framework is non-convex, and the authors relax the optimization for each choice of the cost function using first-order methods. The block coordinate descent algorithm then optimizes the relaxed problem. They show the convergence and accuracy of their proposed algorithm numerically, while a theoretical analysis that guarantees the algorithm's convergence is absent in their work.

\section{Estimating Confidence Intervals of Low-order Moments}
\label{appendix: estimating_low_orders}
In this section, we explain the methodology of estimating confidence intervals for $\mathbb{E}[\bz_i]$ and $\mathbb{E}[\bz_i \bz_j]$. Let $\bX^{n\times d}$ and $\by$ be the data matrix and target variables for $n$ given data points respectively whose entries are in $\Tilde{\mathbb{R}} = \mathbb{R} \cup \{*\}$, where $*$ symbol represents a missing entry. Moreover, assume that $\ba_i$ represents the $i$-th column (feature) of matrix $\bX$. We define: 
\begin{equation*}
    \Tilde{\ba}_i (k) = \begin{cases}
    \ba_i (k) & \text{if} \: \ba_i (k) \neq \text{*} \\
    0 & \text{if} \: \ba_i (k) = \text{*}
    \end{cases}
\end{equation*}
Thus, $\Tilde{\ba}$ is obtained by replacing the missing values with $0$. We estimate the confidence intervals for the mean and covariance of features using multiple bootstrap samples on the available data. Let $\bC_{0} [i][j]$ and $\boldsymbol{\Delta}_{0} [i][j]$ be the center and the radius of the confidence interval for $\bC[i][j]$, respectively. We compute the center of the confidence interval for $\bC[i][j]$ as follows:
\begin{equation} \label{eq: point_estimator}
    \bC_{0} [i][j] = \frac{1}{m_{ij}} \Tilde{\ba}_i^T \Tilde{\ba}_j
\end{equation}
where $m_{i} = |\{k: \ba_i(k) \neq *\}|$ and $m_{ij} = |\{k: \ba_i(k) \neq *, \ba_j(k) \neq *\}|$. This estimator is obtained from the rows where both features are available.  More precisely, let $\bM$ be the mask of the input data matrix $\bX$ defined as:
\begin{equation*}
    \bM_{ij} = \begin{cases}
        0,  & \textrm{if } \bX_{ij} \textrm{ is missing,} \\
        1,  & \textrm{otherwise}.
    \end{cases}
\end{equation*}
Assume that $m_{ij} = (\bM^T \bM)_{ij}$, which is the number of rows in the dataset where both features $i$ and $j$ are available. To estimate the confidence intervals for $\bC_{ij}$, we use Algorithm~\ref{alg: estimating_confidence_interval_length}. First, we select multiple ($K$) samples of size $N = m_{ij}$ from the rows where both features are available. Each one of these samples with size $m_{ij}$ is obtained by applying a bootstrap sampler (sampling with replacement) on the $m_{ij}$ rows where both features are available. Then, we compute the second-order moment of two features for each sample. 

To find the radius of confidence intervals for each given pair $(i, j)$ of features, we choose $k$ different bootstrap samples with length $n$ on the rows where both features $i$ and $j$ are available. Then, we compute $\bC_0[i][j]$ of two features in each bootstrap sample. The standard deviation of these estimations determines the radius of the corresponding confidence interval. Algorithm~\ref{alg: estimating_confidence_interval_length} summarizes the required steps for computing the confidence interval radius for the $ij$-th entry of covariance matrix $\boldsymbol{\Delta}$. Note that the confidence intervals for $\boldsymbol{\mu}$ can be computed similarly.
\begin{algorithm}
	\caption{Estimating Confidence Interval Length $\boldsymbol{\Delta}_{ij}$ for Feature $i$ and Feature $j$.} 
	\label{alg: estimating_confidence_interval_length}
	\begin{algorithmic}[1]
	
	    \State \textbf{Input}: $K:$ Number of bootstrap estimations
	    
	    \FOR {$t = 1, \ldots, K$}
	    \State Pick $n$ samples with replacement from the rows where both $i$-th and $j$-th are available.
        \State Let $(\hat{X}_{i1}, \hat{X}_{j1}), \dots, (\hat{X}_{in}, \hat{X}_{jn})$ be the $i$-th and $j$-th features of the selected samples	    
	    \State $C_t = \frac{1}{n} \sum_{r=1}^{n} \hat{X}_{ir} \hat{X}_{jr}$ 
	    
	    \ENDFOR
	    
	    \State $\boldsymbol{\Delta}_{ij} = \textrm{std}(C_1, C_2, \dots, C_K)$
	    
	\end{algorithmic}
\end{algorithm}
Having $\bC_{0}$ and $\boldsymbol{\Delta}$, the confidence interval for the matrix $\bC$ is computed as follows:
\begin{gather*}
    \bC_{\min} = \bC_{0} - c \boldsymbol{\Delta} \\
    \bC_{\max} = \bC_{0} + c \boldsymbol{\Delta},
\end{gather*}

Computing $\bb_{\min}$ and $\bb_{\max}$ can be done in the same manner. 
The hyper-parameter $c$ is defined to control the robustness of the model by tuning the length of confidence intervals. A larger $c$ corresponds to bigger confidence intervals and, thus, a more robust estimator. On the other hand, large values for $c$ lead to very large confidence intervals that can adversely affect the performance of the trained model. 

\begin{remark}
Since the computation of confidence intervals for different entries of the covariance matrix are independent of each other, they can be computed in parallel. In particular, if $\gamma$ cores are available, $\ceil{d/\gamma}$ features (columns of the covariance matrix) can be assigned to each one of the available cores.
\end{remark}

\section{Solving Robust Ridge Regression with the Optimal Convergence Rate}
\label{appendix: Nesterov_Acceleration}
The convergence rate of Algorithm~\ref{alg: Min_Max_Regression} to the optimal solution of Problem~\eqref{robust_linear_regression} can be slow in practice since the algorithm requires to do a matrix inversion for updating $\btheta$ and applying the box constraint to $\bC$ and $\bb$ at each iteration. While we update the minimization problem in closed-form with respect to $\btheta$, we can speed up the convergence rate of the maximization problem by applying Nesterov's acceleration method to function $g(\bb, \bC)$ in~\eqref{robust_linear_regression2}. Since function $g$ is the minimum of convex functions, its gradient with respect to $\bC$ and $\bb$ can be computed using Danskin's theorem. Algorithm~\ref{alg: Min_Max_Regression_Nesterov} describes the steps to optimize Problem~\eqref{robust_linear_regression2} using Nesterov's acceleration method.

\begin{algorithm}
	\caption{Applying the Nesterov's Acceleration Method to Robust Linear Regression} 
	\label{alg: Min_Max_Regression_Nesterov}
	\begin{algorithmic}[1]
	    \State $\bC_0, \bb_0, \boldsymbol{\Delta}, \boldsymbol{\delta}, T$
        \State \textbf{Initialize}: $\bC_1 = \bC_0, \bb_1 = \bb_0, \gamma_0 = 0, \gamma_1 = 1$. 
        
        \FOR {$i = 1, \ldots, T$}
        \State $\gamma_{i+1} = \frac{1 + \sqrt{1 + 4\gamma_i^2}}{2}$
        \State $Y_{\bC_i} = \bC_{i} + \frac{\gamma_i - 1}{\gamma_{i+1}} (\bC_{i} - \bC_{i-1})$
        \State $\bC_{i+1} = \Pi_{\Delta+} \Big(Y_{\bC_i} + \frac{1}{L} \btheta \btheta^T \Big)$
        \State $Y_{\bb_i} = \bb_{i} + \frac{\gamma_i - 1}{\gamma_{i+1}} (\bb_{i} - \bb_{i-1})$
        \State $\bb_{i+1} = \Pi_{\boldsymbol{\delta}} (Y_{\bb_i} - \frac{2\btheta}{L})$
        
        \State Set $\btheta = (\bC_{i+1} + \lambda I)^{-1} \bb_{i+1}$
        \ENDFOR
	\end{algorithmic}
\end{algorithm}

\begin{theorem}\label{thm: Nesterov_Convergence}
Let $(\tilde{\btheta}, \tilde{\bC}, \tilde{\bb})$ be the optimal solution of~\eqref{robust_linear_regression} and $D = \|\bC_0 - \tilde{\bC}\|_F^2 + \| \bb_0 - \tilde{\bb}\|_2^2$. Assume that for any given $\bb$ and $\bC$, within the uncertainty sets described in~\eqref{robust_linear_regression},  $\|\btheta^{*}(\bb, \bC)\| \leq \tau$. Then, Algorithm~\ref{alg: Min_Max_Regression} computes an $\epsilon$-optimal solution of the objective function in ${\cal O}\Big(\sqrt{\frac{D(\tau+1)^2}{\lambda \epsilon}}\Big)$ iterations.
\end{theorem}
\begin{proof}
The proof is relegated to Appendix~\ref{appendix: proofs}.
\end{proof}

\section{Solving the Dual Problem of the Robust Ridge Linear Regression via ADMM}
\label{appendix: admm}
The Alternating Direction Method of Multipliers (ADMM) is a popular algorithm for efficiently solving linearly constrained optimization problems~\citep{gabay1976dual, hong2016convergence}. It has been extensively applied to large-scale optimization problems in machine learning and statistical inference in recent years~\citep{asslander2018low, zhang2018systematic}. Consider the following optimization problem consisting of two blocks of variables $\bx$ and $\by$ that are linearly coupled:

\begin{equation}\label{General_ADMM_Optimization_Problem}
\arraycolsep=1.4pt\def\arraystretch{1.3}
\begin{array}{lll}
    \displaystyle{\min_{\bw, \bz}} \quad &  \displaystyle{f(\bw) + g(\bz)} \\ 
    \st &  \bA\bw + \bB\bz = \bc, 
\end{array}
\end{equation}

The augmented Lagrangian of the above problem can be written as:

\begin{equation}\label{General_ADMM_Optimization_Problem_Augmented}
\arraycolsep=1.4pt\def\arraystretch{1.3}
\begin{array}{lll}
    \displaystyle{\min_{\bw, \bz}} \quad &  \displaystyle{f(\bw) + g(\bz) + \langle \bA \bw + \bB \bz - \bc, \boldsymbol{\lambda} \rangle + \frac{\rho}{2}\| \bA \bw + \bB \bz - \bc\|^2} \\ 
\end{array}
\end{equation}

ADMM schema updates the primal and dual variables iteratively as presented in Algorithm~\ref{alg: General_ADMM}. 
\begin{algorithm}
	\caption{General ADMM Algorithm} 
	\label{alg: General_ADMM}
	\begin{algorithmic}[1]
        \FOR {$t = 1, \ldots, T$}
        
        \State $\bw^{t+1} = \argmin_{\bw} f(\bw) + \langle \bA\bw + \bB\bz^t -\bc, \boldsymbol{\lambda} \rangle + \frac{\rho}{2} \| \bA\bw + \bB\bz^t - \bc\|^2$
        
        \State $\bz^{t+1} = \argmin_{\bz} f(\bw^{t+1}) + \langle \bA\bw^{t+1} + \bB\bz - \bc, \boldsymbol{\lambda} \rangle + \frac{\rho}{2} \| \bA\bw^{t+1} + \bB\bz - \bc\|^2$
        
        \State $\boldsymbol{\lambda}^{t+1} = \boldsymbol{\lambda}^{t} + \rho (\bA\bw^{t+1} + \bB\bz^{t+1} - \bc)$

        \ENDFOR
	\end{algorithmic}
\end{algorithm}

As we mentioned earlier, simultaneous projection of $\bC$ to the set of positive semi-definite matrices and the box constraint $\bC_{\min} \leq \bC \leq \bC_{\max}$ in Algorithm~\ref{alg: Min_Max_Regression} and Algorithm~\ref{alg: Min_Max_Regression_Nesterov} is computationally expensive. Moreover, careful step-size tuning is necessary to avoid inconsistency and guarantee convergence in that algorithm.

An alternative approach for solving Problem~\eqref{robust_linear_regression} that avoids removing the PSD constraint in the implementation of Algorithm~\ref{alg: Min_Max_Regression} and Algorithm~\ref{alg: Min_Max_Regression_Nesterov} is to solve the dual of the inner maximization problem.  
Since the maximization problem is concave with respect to $\bC$ and $\bb$, and the relative interior of the feasible set of constraints is non-empty, the duality gap is zero. Hence, instead of solving the inner maximization problem, we can solve its dual which is a minimization problem. Theorem~\ref{thm: Dual} describes the dual problem of the inner maximization problem in ~\eqref{robust_linear_regression}. Thus, Problem~\eqref{robust_linear_regression} can be alternatively formulated as a minimization problem rather than a min-max problem. We can solve such a constrained minimization problem efficiently via the ADMM algorithm. As we will show, the ADMM algorithm applied to the dual problem does not need tuning of step-size or applying simultaneous projections to the box constraints and positive semi-definite (PSD) constraints. 
\begin{theorem}\label{thm: Dual} (Dual Problem)
The inner maximization problem described in~\eqref{robust_linear_regression} can be equivalently formulated as:
\begin{equation*}\label{eq: robust_linear_regression_dual1}
\arraycolsep=1.4pt\def\arraystretch{1.3}
\begin{array}{lll}
    \displaystyle{\min_{\bA, \bB, \bd, \be, \bH}} \quad &  \displaystyle{- \langle \bb_{\min}, \bd \rangle + \langle \bb_{\max}, \be \rangle -  \langle \bC_{\min}, \bA \rangle + \langle \bC_{\max}, \bB \rangle + \lambda \| \btheta\|^2}\\
    \st &  - \btheta \btheta^T - \bA + \bB - \bH = 0, \\
        & 2\btheta - \bd + \be = 0, \\
        & \bA, \bB, \bd, \be \geq 0, \\
        & \bH \succeq 0.
\end{array}
\end{equation*}
Therefore, Problem~\eqref{robust_linear_regression} can be alternatively written as:
\begin{equation}\label{eq: robust_linear_regression_dual}
\arraycolsep=1.4pt\def\arraystretch{1.3}
\begin{array}{lll}
    \displaystyle{\min_{\btheta, \bA, \bB, \bd, \be, \bH}} \quad &  \displaystyle{- \langle \bb_{\min}, \bd \rangle + \langle \bb_{\max}, \be \rangle -  \langle \bC_{\min}, \bA \rangle + \langle \bC_{\max}, \bB \rangle + \lambda \| \btheta\|^2}\\
    \st &  - \btheta \btheta^T - \bA + \bB - \bH = 0, \\
        & 2\btheta - \bd + \be = 0, \\
        & \bA, \bB, \bd, \be \geq 0, \\
        & \bH \succeq 0.
\end{array}
\end{equation}
\end{theorem}
\begin{proof}
The proof is relegated to Appendix~\ref{appendix: proofs}.
\end{proof}
To apply the ADMM method to the dual problem, we require to divide the optimization variables into two blocks as in~\eqref{General_ADMM_Optimization_Problem} such that both sub-problems in Algorithm~\ref{alg: General_ADMM} can be efficiently solved. To do so, first, we introduce the auxiliary variables $\bd^{'}, \be^{'}, \btheta^{'}, \bA^{'}$ and $\bB^{'}$ to the dual problem. Also, let $\bG = \bH + \btheta^{'} \btheta^{'T}$. Therefore, Problem~\eqref{eq: robust_linear_regression_dual} is equivalent to:
\begin{equation}\label{robust_linear_regression_dual_admm_auxiliary_variables}
\arraycolsep=1.4pt\def\arraystretch{1.3}
\begin{array}{lll}
    \displaystyle{\min_{\btheta, \bA, \bB, \bd, \be, \bH}} \quad &  \displaystyle{- \langle \bb_{\min}, \bd \rangle + \langle \bb_{\max}, \be \rangle -  \langle \bC_{\min}, \bA \rangle + \langle \bC_{\max}, \bB \rangle + \lambda \| \btheta\|^2}\\
    \st &  \bB - \bA = \bG, \\
        & 2\btheta - \bd + \be = 0, \\
        & \bA = \bA^{'}, \bB = \bB^{'}, \\
        & \bd = \bd^{'}, \be = \be^{'}, \btheta = \btheta^{'}, \\
        & \bA^{'}, \bB^{'}, \bd^{'}, \be^{'} \geq 0, \\
        & \bG \succeq \btheta^{'} \btheta^{'T}.
\end{array}
\end{equation}
Since handling both constraints on $\btheta$ in Problem~\eqref{eq: robust_linear_regression_dual} is difficult, we interchange $\btheta$ with $\btheta^{'}$ in the first constraint. Moreover, the non-negativity constraints on $\bA, \bB, \bd$ and $\be$ are exchanged with non-negativity constraints on $\bA^{'}, \bB^{'}, \bd^{'}$ and $\be^{'}$. 
For the simplicity of presentation, assume that $\bc_1^t = \bb_{\min} - \bmu_d^t + \rho \bd^{'t} + \boldeta^t$, $\bc_2^t = -\bb_{\max} - \bmu_e^t + \rho \be^{'t} - \boldeta^t$, $\bc_3^t = -\bmu_{\btheta} +  \rho \btheta^{'t} - 2 \boldeta^t$, $\bD_1^t = \rho \bA^{'t} - \rho \bG^t + \bgamma^t - \bM_A^t + \bC_{\min}$, and $\bD_2^t = \rho \bB^{'t} + \rho \bG^t - \bgamma^t - \bM_B^t - \bC_{\max}$. Algorithm~\ref{alg: solving_dual_admm} describes the ADMM algorithm applied to Problem~\eqref{robust_linear_regression_dual_admm_auxiliary_variables}.
\begin{corollary}\label{cor: ADMM} If the feasible set of Problem~\eqref{robust_linear_regression} has non-empty interior, then Algorithm~\ref{alg: solving_dual_admm} converges to an $\epsilon$-optimal solution of Problem~\eqref{robust_linear_regression_dual_admm_auxiliary_variables} in $\mathcal{O}(\frac{1}{\epsilon})$ iterations.
\end{corollary}
\begin{proof}
Since the inner maximization problem, in~\eqref{robust_linear_regression} is convex, and its feasible interior set is not empty, the duality gap is zero by Slater's condition. Thus, according to Theorem~$6.1$ in~\citet{he2015non}, Algorithm~\ref{alg: solving_dual_admm} converges to an optimal solution of the primal-dual problem with a linear rate. Moreover, the sequence of constraint residuals converges to zero with a linear rate as well. 
\end{proof}
\begin{remark}
The optimal solution obtained from the ADMM algorithm can be different from the one given by Algorithm~\ref{alg: Min_Max_Regression} because we remove the positive semi-definite constraint on $\bC$ in the latter. We investigate the difference between solutions of two algorithms in three cases: First, we generate a small positive semi-definite matrix $\bC^*$ and the matrix of confidence intervals ($\boldsymbol{\Delta}$) as follows:
\begin{equation*}
    \bC^* = \begin{bmatrix}
    97 & 40 & 92 \\
    40 & 17 & 38 \\
    92 & 38 & 88
  \end{bmatrix}, \quad \boldsymbol{\Delta} = \begin{bmatrix}
    0.2 & 0.3 & 0.2 \\
    0.3 & 0.1 & 0.2 \\
    0.1 & 0.3 & 0.1
  \end{bmatrix}. 
\end{equation*}
Moreover, let $\bb^*$ and $\boldsymbol{\delta}$ are generated as follows:
\begin{equation*}
    \bb^* = \begin{bmatrix}
    6.65 \\
    8.97 \\
    5.40
  \end{bmatrix}, \quad \boldsymbol{\delta} = \begin{bmatrix}
    0.1 \\
    0.2 \\
    0.2
  \end{bmatrix}. 
\end{equation*}
Initializing both algorithms with a random matrix within $\bC_{\min} = \bC^* - \boldsymbol{\Delta}$ and $\bC_{\max} = \bC^* + \boldsymbol{\Delta}$, and a random vector within $\bb_{\min} = \bb^* - \boldsymbol{\delta}$ and $\bb_{\max} = \bb^* + \boldsymbol{\delta}$, ADMM algorithm returns a different solution from Algorithm~\ref{alg: Min_Max_Regression}. 
Besides, the difference in the performance of algorithms during the test phase can be observed in the experiments on synthetic datasets depicted in Figure~\ref{fig: consistency_different_methods} as well, especially when the number of samples is smaller. 
\end{remark}
\begin{algorithm}[H]
	\caption{Applying ADMM to the Dual Reformulation of Robust Linear Regression} 
	\label{alg: solving_dual_admm}
	\begin{algorithmic}[1]
	    \State \textbf{Given:} $\bb_{\min}, \bb_{\max}, \bC_{\min}, \bC_{\max}, \lambda, \rho$
        \State \textbf{Initialize}: $\bC_1 = \bC_0, \bb_1 = \bb_0, \gamma_0 = 0, \gamma_1 = 1$. 
        
        \FOR {$t = 0, \ldots, T$}
        \State $\btheta^{t+1} = \frac{1}{6\lambda + 7\rho} (2\bc_1^t - 2\bc_2^t - 3\bc_3^t)$
        
        \State $\bd^{t+1} = \frac{1}{6\lambda + 7\rho} (\frac{6\rho + 4\lambda}{\rho} \bc_1^t + \frac{4\rho + 4\lambda}{\rho} \bc_2^t + 2\bc_3^t)$
        
        \State $\be^{t+1} = \frac{2}{6\lambda + 7\rho} (\frac{\rho + 2\lambda}{\rho} \bc_1^t + \frac{3\rho + 2\lambda}{\rho} \bc_2^t - \bc_3^t)$
        
        \State $\bA^{'t+1} = \max(\bA^{t} + \frac{\bM_A^{t}}{\rho}, 0)$
        
        \State $\bB^{' t+1} = \max(\bB^{t} + \frac{\bM_B^{t}}{\rho}, 0)$

        \State $\bG^{t+1} = [\bB^{t} - \bA^{t} + \frac{\bgamma ^{t}}{\rho} - \btheta^{'t} \btheta^{'t T} ]_+ + \btheta^{'t} \btheta^{'t T}$

        \State $\bd^{' t+1} = \max(\bd^{t} + \frac{\bmu_d^t}{\rho}, 0)$
        
        \State $\be^{' t+1} = \max(\be^{t} + \frac{\bmu_e^{t}}{\rho}, 0)$
        
        \State $\btheta^{' t+1} = \argmin_{\btheta^{'}} \quad \| \btheta^{t+1} -  \btheta^{'}\|^2 + \langle \bmu_{\btheta}^t, \btheta^{t+1} - \btheta^{'} \rangle \quad \st \: \bG^{t+1} \succeq \btheta^{'T} \btheta^{'}$ 
        
        \State $\bA^{t+1} = \frac{1}{3\rho} (2\bD_1^t + \bD_2^t)$
        
        \State $\bB^{t+1} = \frac{1}{3\rho} (\bD_1^t + 2\bD_2^t)$
        
        \State $\bM_A^{t+1} = \bM_A^t + \rho (\bA^{t+1} - \bA^{' t+1})$
        \State $\bM_B^{t+1} = \bM_B^t + \rho (\bB^{t+1} - \bB^{' t+1})$
        \State $\bmu_d^{t+1} = \bmu_d^t + \rho (\bd^{t+1} - \bd^{' t+1})$
        \State $\bmu_e^{t+1} = \bmu_e^t + \rho (\be^{t+1} - \be^{' t+1})$
        \State $\bmu_{\btheta}^{t+1} = \bmu_{\theta}^t + \rho ({\btheta}^{t+1} - {\btheta}^{' t+1})$
        
        \State $\boldeta^{t+1} = \boldeta^t + \rho (2\btheta^{t+1} - \bd^{t+1} + \be^{t+1})$
        
        \State  $\bgamma^{t+1} = \bgamma^t + \rho (\bB^{t+1} - \bA^{t+1} - \bG^{t+1})$
        \ENDFOR
	\end{algorithmic}
\end{algorithm}
Now, we show how to apply ADMM schema to Problem~\eqref{robust_linear_regression_dual_admm_auxiliary_variables} to obtain Algorithm~\ref{alg: solving_dual_admm}. As we discussed earlier, we consider two separate blocks of variables $\bw = (\btheta, \bd, \be, \bG, \bB^{'}, \bA^{'})$  and $\bz = (\bd^{'}, \be^{'}, \btheta^{'}, \bB, \bA)$. Assigning $\boldsymbol{\Gamma}, \boldeta, \bM_{A}, \bM_{B}, \bmu_{d}, \bmu_{e},$ and $\bmu_{\theta}$ to the constraints of Problem~\eqref{robust_linear_regression_dual_admm_auxiliary_variables} in order, we can write the corresponding augmented Lagrangian function as:
\begin{equation}\label{robust_linear_regression_dual_augmented_lagrangian}
\arraycolsep=1.4pt\def\arraystretch{1.3}
\begin{array}{lll}
    \displaystyle{\min_{\btheta, \btheta^{'}, \bA, \bA^{'}, \bB, \bB^{'}, \bd, \bd^{'}, \be, \be^{'}, \bG}} \quad &  \displaystyle{- \langle \bb_{\min}, \bd \rangle + \langle \bb_{\max}, \be \rangle -  \langle \bC_{\min}, \bA \rangle + \langle \bC_{\max}, \bB \rangle + \lambda \| \btheta\|^2}\\
    & + \langle \bA - \bA^{'}, \bM_A \rangle + \frac{\rho}{2} \| \bA - \bA^{'}\|_F^2 \\
    & + \langle \bB - \bB^{'}, \bM_B \rangle + \frac{\rho}{2} \| \bB - \bB^{'}\|_F^2 \\
    & +  \langle \bd - \bd^{'}, \bmu_d \rangle + \frac{\rho}{2} \| \bd - \bd^{'}\|^2 \\
    & + \langle \be - \be^{'}, \bmu_e \rangle + \frac{\rho}{2} \| \be - \be^{'}\|^2 \\
    & + \langle \btheta - \btheta^{'}, \bmu_{\theta} \rangle + \frac{\rho}{2} \| \btheta - \btheta^{'}\|^2 \\
    & + \langle 2\btheta - \bd + \be, \boldeta \rangle + \frac{\rho}{2} \| 2\btheta - \bd + \be\|^2 \\
    & + \langle \bB - \bA - \bG, \boldsymbol{\Gamma} \rangle + \frac{\rho}{2} \| \bB - \bA - \bG\|_F^2 \\ 
    \st &  \bA^{'}, \bB^{'}, \bd^{'}, \be^{'} \geq 0, \\
        & \mathbf{G} \succeq \btheta^{'} \btheta^{'T},
\end{array}
\end{equation}
At each iteration of the ADMM algorithm, the parameters of one block are fixed, and the optimization problem is solved with respect to the parameters of the other block. For the simplicity of presentation, let $\bc_1^t = \rho \btheta^{'t} - \bmu_{\btheta}^t - 2\boldeta^t, \bc_2^t = \rho \bd^{'t} - \bmu_{\bd}^t - \bb_{\min} + \boldeta^t, \bc_3^t = \rho \be^{'t} - \bmu_{e}^t + \bb_{\max} - \boldeta^t, \bD_1^t = \rho \bA^{'t} - \rho \bG^t + \boldsymbol{\Gamma}^t - \bM_A^t + \bC_{\min},$ and $\bD_2^t = \rho \bB^{'t} + \rho \bG^t - \boldsymbol{\Gamma}^t - \bM_B^t - \bC_{\max}$. 

We have two non-trivial problems containing positive semi-definite constraints. The sub-problem with respect to $\bG$ can be written as:
\begin{equation}\label{sub_projection_G}
\arraycolsep=1.4pt\def\arraystretch{1.3}
\begin{array}{lll}
    \displaystyle{\min_{\bG}} \quad &  \displaystyle{ \langle \bB^t - \bA^t - \bG, \boldsymbol{\Gamma}^t \rangle + \frac{\rho}{2} \| \bB^t - \bA^t - \bG\|_F^2
    }\\
    \st & \bG \succeq \btheta^{'t} \btheta^{'tT},
\end{array}
\end{equation}
By completing the square, and changing the variable $\bG^{'} = \bG - \btheta^{'t} \btheta^{'tT}$, equivalently we require to solve the following problem:
\begin{equation}\label{sub_problem_G_prime}
\arraycolsep=1.4pt\def\arraystretch{1.3}
\begin{array}{lll}
    \displaystyle{\min_{\bG^{'}}} \quad &  \displaystyle{\frac{\rho}{2} \| \bG^{'} - (\bB^t - \bA^t - \btheta^{'t} \btheta^{'tT} + \frac{\boldsymbol{\Gamma}^t}{\rho})\|_F^2
    }\\
    \st & \boldsymbol{G}^{'} \succeq 0,
\end{array}
\end{equation}
Thus, $\bG^{'*} = [\bB^{t} - \bA^{t} + \frac{\boldsymbol{\Gamma} ^{t}}{\rho} - \btheta^{'t} \btheta^{'t T} ]_+$, where $[\bA]_+$ is the projection to the set of PSD matrices, which can be done by setting the negative eigenvalues of $\bA$ in its singular value decomposition to zero.

The other non-trivial sub-problem in Algorithm~\eqref{alg: solving_dual_admm} is the minimization with respect to $\btheta^{'}$ (Line 10). By completing the square, it can be equivalently formulated as: 
\begin{equation}\label{projection_psd}
\arraycolsep=1.4pt\def\arraystretch{1.3}
\begin{array}{lll}
    \displaystyle{\min_{\btheta^{'}}} \quad &  \displaystyle{\| \btheta^{'} - (\btheta^{t+1} + \frac{\bmu_{\btheta}^{t}}{\rho}) \|_2^2}\\
    \st &  \bG^{t+1} \succeq \btheta^{'} \btheta^{'T},
\end{array}
\end{equation}
Let $\bG = \boldsymbol{U} \boldsymbol{\Lambda} \boldsymbol{U}^T$ be the singular value decomposition of the matrix $\bG$ where $\boldsymbol{\Lambda}$ is a diagonal matrix containing the eigenvalues of the matrix G. Set $\boldsymbol{\alpha} = \btheta^{t+1} + \frac{\bmu_{\btheta}^{t}}{2}$. Since $\boldsymbol{U}^T \boldsymbol{U} = \boldsymbol{I}$, we have: 
\begin{gather*}
    \| \boldsymbol{U}^T \btheta - \boldsymbol{U}^T \boldsymbol{\alpha} \|^2 = \| \btheta - \boldsymbol{\alpha} \|_2^2
\end{gather*}
Set $\boldsymbol{\beta} = \boldsymbol{U}^T \btheta^{'}$, then Problem~\eqref{projection_psd} can be reformulated as:
\begin{equation}\label{projection_psd_reformulation}
\arraycolsep=1.4pt\def\arraystretch{1.3}
\begin{array}{lll}
    \displaystyle{\min_{\boldsymbol{\beta}}} \quad &  \displaystyle{\| \boldsymbol{\beta} - \boldsymbol{U}^T  \boldsymbol{\alpha}\|_2^2}\\
    \st &  \boldsymbol{\beta} \boldsymbol{\beta}^T \preceq \boldsymbol{\Lambda}.
\end{array}
\end{equation}
Note that the constraint of the above optimization problem is equivalent to the following:
\begin{gather*}
\boldsymbol{\beta} \boldsymbol{\beta}^T \preceq \boldsymbol{\Lambda} \Leftrightarrow \begin{bmatrix}
1 & \boldsymbol{\beta}^T \\
\boldsymbol{\beta} & \boldsymbol{\Lambda}
\end{bmatrix} \succeq 0 \Leftrightarrow \boldsymbol{\beta}^T \boldsymbol{\Lambda}^{-1} \boldsymbol{\beta} \leq 1 \Leftrightarrow \sum_{i=1}^{n} \frac{\boldsymbol{\beta}_{i}^2}{\lambda_i} \leq 1,
\end{gather*}
where $\lambda_i = \boldsymbol{\Lambda}_{ii}$   
Since the block matrix is symmetric, using Schur Complement, it is positive semi-definite if and only if $\boldsymbol{\Lambda}$ is positive semi-definite and $ 1 - \boldsymbol{\beta}^t \boldsymbol{\Lambda}^{-1} \boldsymbol{\beta} \geq 0$ (the third inequality above). 

Set $\boldsymbol{\gamma} = \boldsymbol{U}^T \boldsymbol{\alpha}$, then we can write Problem~\eqref{projection_psd_reformulation} as:
\begin{equation}\label{projection_psd_final_problem}
\arraycolsep=1.4pt\def\arraystretch{1.3}
\begin{array}{lll}
    \displaystyle{\min_{\boldsymbol{\beta}}} \quad &  \displaystyle{\| \boldsymbol{\beta} - \boldsymbol{\gamma} \|_2^2}\\
    \st &  \sum_{i=1}^{n} \frac{\boldsymbol{\beta}_{i}^2}{\lambda_i} \leq 1,
\end{array}
\end{equation}
It can be easily shown that the optimal solution has the form $\boldsymbol{\beta}^*_i = \frac{\boldsymbol{\gamma}_i}{1 + \frac{\bmu^{*}}{\lambda_i}}$, where $\bmu^{*}$ is the optimal Lagrangian multiplier corresponding to the constraint of Problem~\eqref{projection_psd_final_problem}. The optimal Lagrangian multiplier can be obtained by the bisection algorithm similar to Algorithm~\ref{alg: bisection}. Having $\boldsymbol{\beta}^{*}$, the optimal $\btheta$ can be computed by solving the linear equation $\boldsymbol{U}^T \btheta^* = \boldsymbol{\beta}^*$. 

\section{Quadratic RIFLE: Using Kernels to Go Beyond Linearity}
\label{appendix: qrifle}
A natural extension of RILFE to non-linear models is to transform the original data via multiple Kernels and then apply RIFLE to the obtained data. To this end, we applied Polynomial Kernels to the original data that considers the polynomial transformations of features and their interactions. A drawback of this approach is that if the original data contains $d$ features, and the order of the polynomial Kernel is $t$, the number of features in the transformed data will be $\mathcal{O} (d^t)$ that increases the runtime of the prediction/imputation drastically. Thus, we only consider $t = 2$, which leads to a dataset containing the interaction of different features of the original data. We call the RIFLE algorithm applied on the data transformed by Quadratic Kernel Quadratic RIFLE (QRIFLE). Table~\ref{tab:imputation_uci} demonstrates the performance of QRIFLE alongside RIFLE and other state-of-the-art approaches. Moreover, we applied QRIFLE on a regression task where the correlation between predictors and the target variable is quadratic (Figure~\ref{fig: qrifle}). We can observe that QRIFLE works better than RIFLE when the percentage of missing values is not high.

\subsection{Performance of RIFLE and QRIFLE on Synthetic Non-linear Data}
A natural question is how RIFLE performs when the underlying model is non-linear. To evaluate RIFLE and other methods, we have generated jointly normal data similar to the experiment in Figure~\ref{fig: consistency}. Here, we have $5000$ data points, and the data dimension is $d = 5$. The target variable has the following quadratic relationship with the input features: 
\begin{equation*}
y = x_1^2 + 3x_3^2 - 6x_5^2 - 0.9 x_1 x_4 + 9 x_2 x_3 + 3.2 x_4 x_5 - 1.7 x_2 x_5 - 5 x_1 - 2x_3 + 7x_4 + 4.6    
\end{equation*}

\vspace{-3mm}
\begin{figure}[ht]
\centering
\includegraphics[width=0.99\columnwidth]{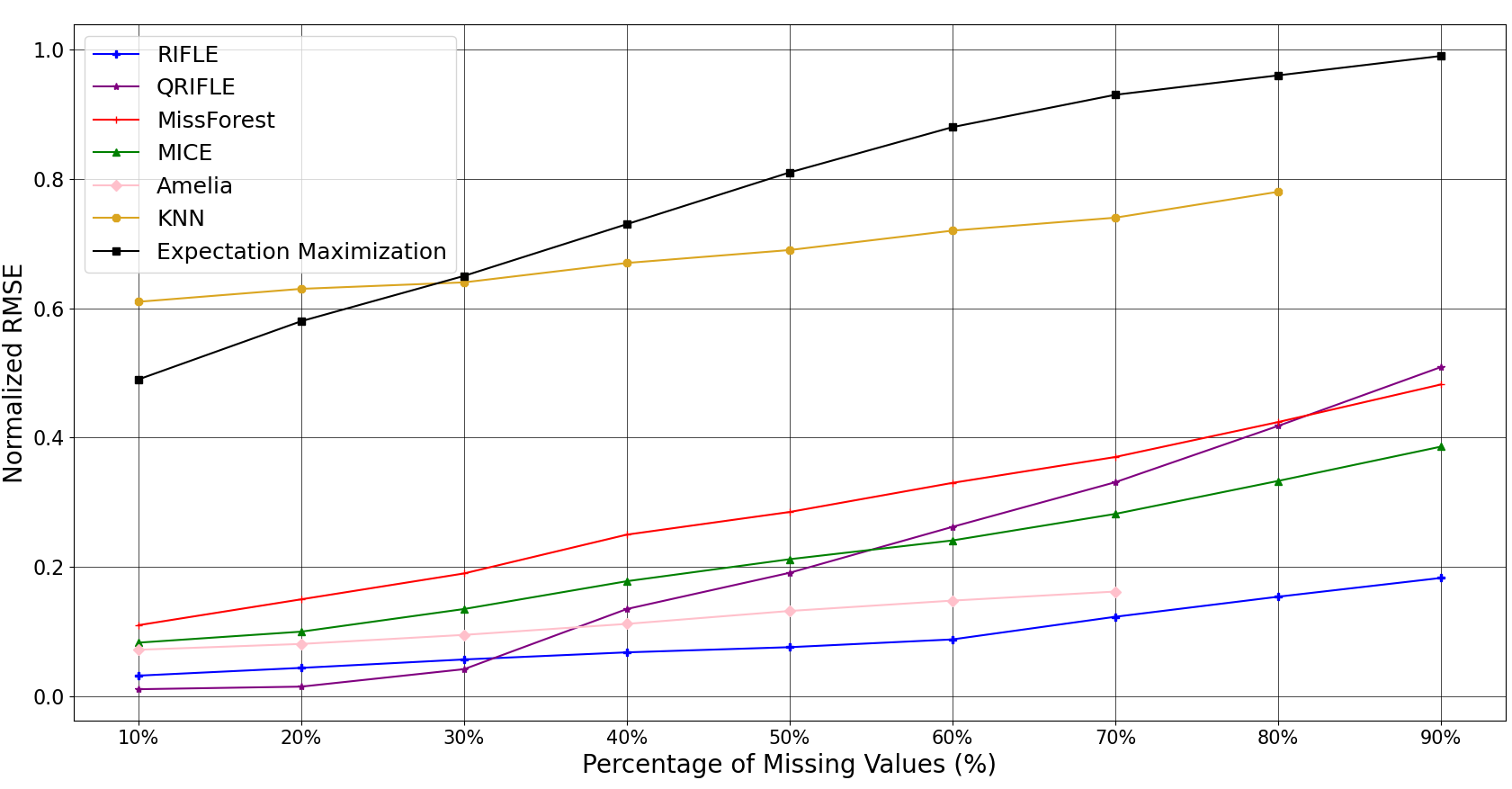}
\vspace{-2mm}
\caption{Performance of RIFLE, QRIFLE, MissForest, Amelia, KNN Imputer, MICE, Expectation Maximization to the percentage of missing values on Quadratic artificial datasets with different percentages of missing values.}
\label{fig: qrifle}
\end{figure}

We evaluated the performance of KNN-Imputer~\citep{troyanskaya2001missing}, MICE~\citep{buuren2010mice}, Amelia~\citep{honaker2011amelia}, MissForest~\citep{stekhoven2012missforest}, and Expectation Maximization~\citep{dempster1977maximum}, alongside the RIFLE and QRIFLE. QRIFLE is the RIFLE application on the original data transformed by a polynomial kernel with the degree of $2$. Although QRIFLE can learn the quadratic models, the number of missing values in the new features (interaction terms) will be higher than the original data. For instance, if, on average, $50\%$ of entries are missing in the original features, there will be $75\%$ of missing entries in the interaction terms. Moreover, the computation complexity will be increased since we have $d^2$ features instead of $d$ if we use QRIFLE. Figure~\ref{fig: qrifle} demonstrates the performance of the aforementioned methods on the artificial data with $5000$ samples containing different percentages of missing values. We generated $5$ artificial datasets for each missing value percentage, and each method is performed $5$ times on the datasets. We reported the average performances for each method in Figure~\ref{fig: qrifle}. For small percentages of missing values, QRIFLE performs better than other approaches. However, by increasing the percentage of missing values, QRIFLE performance drops, and RIFLE works much better than RIFLE. 
\section{Robust Quadratic Discriminant Analysis (Presence of Missing Values in the Target Feature)}
\label{appendix: robust_classification}
In Section~\ref{sec: robust_classification} we formalized robust quadratic discriminant analysis assuming the target variable is fully available. In this appendix, we study Problem~\eqref{robust_logistic_regression} when the target variable contains missing values.

 If the target feature contains missing values, the proposed algorithm for solving the optimization problem~\eqref{eq: robust_logistic_regression_finite} does not exploit the data points whose target feature is unavailable. However, such points can contain valuable statistical information about the underlying data distribution. Thus, we apply an Expectation Maximization (EM) procedure on the dataset as follows:

Assume that a dataset consisting of $n+m$ samples. Let $(\bX_1, y_1), \dots, (\bX_n, y_n)$ be $n$ samples whose target variable is available and $(\bX_{n+1}, z_1), \dots, (\bX_{n+m}, z_m)$ are $m$ samples where their corresponding labels are missing. Similar to the previous case, we assume:
\begin{equation*}
\bX_i | z_i = j \sim \mathcal{N}(\bmu_j, \boldsymbol{\Sigma}_j), \quad j = 0, 1.    
\end{equation*}

Thus, the probability of observing a data point $x_i$ can be written as:

\begin{gather*}
    P (\bX_i = \bx_i) = \pi_0 P(\bX_i = \bx_i | z_i = 0) + \pi_1 P(\bX_i = \bx_i | \bz_i = 1) \\
    = \pi_0 \mathcal{N} (\bx_i; \mu_0, \boldsymbol{\Sigma}_0) + \pi_1 \mathcal{N} (\bx_i; \bmu_1, \boldsymbol{\Sigma}_1) 
\end{gather*}

The log of likelihood function can be formulated as follows:

\begin{equation*}
    \ell(\bmu_0, \boldsymbol{\Sigma}_0, \bmu_1, \boldsymbol{\Sigma}_1) = \sum_{i=1}^{n+m} \log \Big( \pi_0 \mathcal{N} (\bx_i; \bmu_0, \boldsymbol{\Sigma}_0) + \pi_1 \mathcal{N} (\bx_i; \bmu_1, \boldsymbol{\Sigma}_1) \Big)
\end{equation*}

We apply Expectation Maximization procedure to jointly update $\boldsymbol{\Sigma}_0, \boldsymbol{\Sigma}_1, \bmu_0, \bmu_1$ and $z_i$'s. Note that the posterior distribution of $z_i$ can be written as:

\begin{equation*}
    P (Z_i = t | \bX_i = \bx_i) = \frac{P(\bX_i = \bx_i | Z_i = t) P(Z_i = t)}{P(\bX_i = \bx_i)} = \frac{\pi_t \mathcal{N} (\bx_i; \bmu_t, \boldsymbol{\Sigma}_t)}{P(\bX_i = \bx_i)}
\end{equation*}

We update $z_i$ values in the E-step by comparing the posterior probabilities for two possible labels. Precisely, we assign label $1$ to $Z_i$ if and only if:
\begin{equation*}
    \pi_1 \mathcal{N} (\bx_i; \bmu_1, \boldsymbol{\Sigma}_1) > \pi_0 \mathcal{N} (\bx_i; \bmu_0, \boldsymbol{\Sigma}_0)
\end{equation*}

In M-step, we estimate $\boldsymbol{\Sigma}_0, \boldsymbol{\Sigma}_1, \bmu_0, \bmu_1, \pi_0$ and $pi_0$ by fixing the $z_i$ values. Since in M-step, all labels (both already available $y_i$'s and estimated $z_i$'s in E-step) are assigned, updating the aforementioned parameters can be done as follows:
\begin{gather}
    \bmu_1[j] = \frac{1}{|\mathcal{S}_1 \cap \mathcal{T}_j|}\sum_{i \in \mathcal{S}_1 \cap \mathcal{T}_j} \bx_i[j] \\
    \bmu_0[j] = \frac{1}{|\mathcal{S}_0 \cap \mathcal{T}_j|}\sum_{i \in \mathcal{S}_0 \cap \mathcal{T}_j} \bx_i[j] \\
    \boldsymbol{\Sigma}_1[i][j] = \frac{1}{|\mathcal{S}_1 \cap \mathcal{T}_i \cap \mathcal{T}_j|}\sum_{t \in \mathcal{S}_1 \cap \mathcal{T}_i \cap \mathcal{T}_j} \bx_t[i] \bx_t [j] \\
    \boldsymbol{\Sigma}_0[i][j] = \frac{1}{|\mathcal{S}_0 \cap \mathcal{T}_i \cap \mathcal{T}_j|}\sum_{t \in \mathcal{S}_0 \cap \mathcal{T}_i \cap \mathcal{T}_j} \bx_t[i] \bx_t [j]\\
    \pi_1 = \frac{|S_1|}{|S_0 \cup S_1|}\\
    \pi_0 = \frac{|S_0|}{|S_0 \cup S_1|}
\end{gather}
We apply the M-step and E-step iteratively to obtain $\boldsymbol{\Sigma}_1$ and $\boldsymbol{\Sigma}_0$. Based on the random initialization of $z_i$'s we can obtain different values for $\boldsymbol{\mu}_0, \boldsymbol{\mu}_1, \boldsymbol{\Sigma}_0$ and $\boldsymbol{\Sigma}_1$. Having these estimations, we apply Algorithm~\ref{alg: robust_normal_discriminant analysis} to solve the robust normal discriminant analysis formulated in~\eqref{eq: robust_logistic_regression_finite}.
\begin{algorithm}
	\caption{Expectation Maximization Procedure for Learning a Robust Normal Discriminant Analysis} 
	\label{alg: EM_Robust_Logistic_Regression}
	\begin{algorithmic}[1]
        \State \textbf{Input}: $T$: Number of EM iterations, $k$: Number of covariance estimations at each iteration. 
        
        \State \textbf{Initialize}: Set each missing labels randomly to $0$ or $1$.  
        
        \FOR {$i = 1, \ldots, T$}
    
            \State Estimate $k$ covariance matrices by sampling with replacement from the available entries
            
            \State Find an optimal $\bw$ for Problem~\eqref{eq: robust_logistic_regression_finite_linear_combination}
            
            \State Update the missing labels using the new $\bw$ obtained above.
        \ENDFOR
        
	\end{algorithmic}
\end{algorithm}

\section{Generating Missing Values Patterns in Numerical Experiments}
\label{appendix: MNAR}
In this appendix, we define MCAR and MNAR patterns and discuss how to generate them in a given dataset. Formally, the distribution of missing values in a dataset follows a missing completely at random (MCAR) pattern if the probability of having a missing value for a given entry is constant, independent of other available and missing entries. On the other hand, a dataset follows a Missing At Random (MAR) pattern if the missingness of each entry only depends on the available data of other features. Finally, if the distribution of missing values does not follow an MCAR or MAR pattern, we call it missing not at random (MNAR). 

To generate the MCAR pattern on a given dataset, we fix a constant probability $0 < p < 1$ and make each data entry unavailable with the probability of $p$. On the other hand, the generation of the MNAR pattern is based on the idea that if the value of an entry is farther from the mean of its corresponding feature, then the probability of missingness for that entry is larger. 

The generation of the MNAR pattern is based on the idea that if the value of an entry is farther from the mean of its corresponding feature, then the probability of missingness for that entry is larger. Algorithm~\ref{alg: MNAR_Generator} describes the procedure of generating MNAR missing values for a given column of a dataset:
\begin{algorithm}[H]
	\caption{Generating MNAR Pattern for a Given Column of a Dataset} 
	\label{alg: MNAR_Generator}
	\begin{algorithmic}[1]
        \State \textbf{Input}: $x_1, x_2, \dots, x_n$: The entries of the current column in the dataset, $a, b$: Hyper-parameters controlling the percentage of missing values
        
        \State \textbf{Initialize}: Set $\mu = \frac{1}{n} \sum_{i=1}^n x_i$ and $\sigma^2 = \frac{1}{n} \sum_{i=1}^n x_i^2 - \mu^2$.  
        
        \FOR {$i = 1, \ldots, n$}
    
            \State $x^{'}_i = \frac{x_i - \mu}{\sigma}$
            
            \State $p_i = F(a |x^{'}_i| + b)$
            
            \State Set $x_i = *$ with probability of $p_i$
            
        \ENDFOR
	\end{algorithmic}
\end{algorithm}
Note that $F$ in the above algorithm is the cumulative distribution function of a standard Gaussian random variable. $a$ and $b$ control the percentage of missing values in the given column. As $a$ and $b$ increase, the probability of having more missing values is higher. Since the availability of each data entry depends on its value, the generated missing pattern is missing not at random (MNAR).

\vspace{-3mm}
\section{Proof of Lemmas and Theorems}
\vspace{-3mm}
\label{appendix: proofs}
In this appendix, we prove all lemmas and theorems presented in the article. First, we prove the following lemma that is useful in several convergence proofs:
\begin{lemma}\label{lemma: Lipschitz constant}
Let $\btheta^{*}(\bb, \bC) = \argmin_{\btheta} \btheta^T \bC \btheta - 2 \bb^T \btheta + \lambda \| \btheta \|^2$. Assume that for any given $\bb$ and $\bC$, $\|\btheta^{*}(\bb, \bC)\| \leq \tau$. Then, the Lipschitz constant of the gradient of the function $g(\bb, \bC) = \min_{\btheta} \btheta^T \bC \btheta - 2 \bb^T \btheta + \lambda \| \btheta \|^2$ used in Problem~\eqref{robust_linear_regression2} is  equal to
 $L = \frac{2(\tau+1)^2}{\lambda}$.
 \end{lemma}
\begin{proof} Since the problem is convex in $\btheta$ and concave in $\bC$ and $\bb$, we have:
\begin{gather*}
\min_{\btheta} \quad \max_{\bC, \bb} \quad \btheta^T \bC \btheta - 2 \bb^T \btheta + \lambda \| \btheta \| ^2 = - \min_{\bC, \bb} \quad \max_{\btheta} \quad - \btheta^T \bC \btheta + 2 \bb^T \btheta - \lambda \| \btheta \| ^2     
\end{gather*}
Assume that $h(\btheta, \bC, \bb) \triangleq - \btheta^T \bC \btheta + 2 \bb^T \btheta - \lambda \| \btheta \| ^2$. Define $L_{11}$, $L_{12}$ as follows:
\begin{gather*}
    \|\nabla_{\bb, \bC} h(\btheta, \bb_1, \bC_1) - \nabla_{\bb, \bC} h(\btheta, \bb_2, \bC_2) \| \leq L_{11} \| (\bC_1, \bb_1) - (\bC_2, \bb_2)\| \\
    \|\nabla_{\theta} h(\btheta, \bb_1, \bC_1) - \nabla_{\btheta} h(\btheta, \bb_2, \bC_2) \| \leq L_{12} \| (\bC_1, \bb_1) - (\bC_2, \bb_2)\|
\end{gather*}
$h(\btheta, \bb, \bC)$ is convex in $\bC$ and $\bb$ and strongly concave with respect to $\btheta$. According to Lemma 1 in~\citet{barazandeh2020solving}, $g^{'} = - g =  \max_{\btheta} h(\btheta, \bb, \bC)$ is Lipschitz continuous with the Lipschitz constant equal to:
\begin{equation*}
    L_{g} = L_{g^{'}} = L_{11} + \frac{L_{12}^2}{\sigma},
\end{equation*}

where $\sigma = 2\lambda$ is the strong-concavity modulus of $- \btheta^T \bC \btheta + 2 \bb^T \btheta - \lambda \| \btheta \| ^2$. 
Note that
\begin{gather*}
    \nabla_{\bb, \bC} \: h(\btheta, \bb, \bC) = - \btheta \btheta^T + 2\btheta \rightarrow \\
    \nabla_{\bb, \bC} h(\btheta, \bb_1, \bC_1) - \nabla_{\bb, \bC} h(\btheta, \bb_2, \bC_2) = 0
\end{gather*}
Thus, $L_{11} = 0$. On the other hand,
\begin{gather*}
    \nabla_{\btheta} \: h(\btheta, \bb, \bC) = -2\btheta \bC + 2\bb - 2\lambda \btheta  \rightarrow \\
    \nabla_{\btheta} h(\btheta, \bb_1, \bC_1) - \nabla_{\btheta} h(\btheta, \bb_2, \bC_2) = -2(\bC_1 - \bC_2)\btheta + 2(\bb_1 - \bb_2) \leq 2\|\bC_1 - \bC_2\|_2 \|\btheta\|_2 + 2\|\bb_1 - \bb_2\|_2 \\ \leq 2\|(\bC_1, \bb1) - (\bC_2, \bb_2)\|_2 \|\btheta\|_2 + 2\|(\bC_1, \bb1) - (\bC_2, \bb_2)\|_2 \leq (2\|\btheta\|_2 + 2) \|(\bC_1, \bb1) - (\bC_2, \bb_2)\|_2  
\end{gather*}
Therefore, $L_{12} = 2 \max \|\btheta \|_2 + 2$, which means $L_g = \frac{2(\max \|\btheta \| + 1)^2}{\lambda}$. Note that $\btheta$ is computed exactly in Algorithm~\ref{alg: Min_Max_Regression} and Algorithm~\ref{alg: Min_Max_Regression_Nesterov} at each iterations. Thus, during the optimization procedure the norm of $\btheta$ is bounded by the maximum norm of $\btheta$ for any given $\bb$ and $\bC$: 
\begin{equation*}
\max \|\btheta\|_2 \leq \max_{\bb, \bC} \btheta^{*}(\bb, \bC) \leq  \tau
\end{equation*}
As a result, $L_g = \frac{2(\tau + 1)^2}{\lambda}$.
\end{proof}

\noindent \textbf{Proof of Theorem}~\ref{thm: Robust_Regression_Convergence}:
Since the set of feasible solutions for $\bb$ and $\bC$ defines a compact set, and function $g$ is a concave function with respect to  $\bb$ and $\bC$, the projected gradient ascent algorithm converges to the global maximizer of $g$ in $T = \mathcal{O}(\frac{LD}{\epsilon})$ iterations \citep[Theorem 3.3]{bubeck2014convex}, where $D = \|\bC_0 - \bC^*\|_F^2 + \| \bb_0 - \bb^{*}\|_2^2$ and $L$ is the Lipschitz constant of function $g$, which is equal to $\frac{2(\tau + 1)^2}{\lambda}$ according to Lemma~\ref{lemma: Lipschitz constant}. \newline

\noindent \textbf{Proof of Theorem}~\ref{thm: Nesterov_Convergence}
Algorithm~\ref{alg: Min_Max_Regression_Nesterov} applies the projected Nesterov acceleration method on the concave function $g$. As proved 
in~\citet{nesterov1983method}, the rate of convergence of this method conforms to the lower bound of first-order oracles for the general convex minimization (concave maximization) problems, which is ${\cal O}(\sqrt{\frac{LD^2}{\epsilon}})$. We compute the Lipschitz constant $L$ that appeared in the iteration complexity bound by Lemma~\ref{lemma: Lipschitz constant}. \newline

\noindent \textbf{Proof of Theorem}~\ref{thm: Dual}:
First, note that if we multiply the objective function by $-1$, Problem~\eqref{robust_linear_regression} can be equivalently formulated as:
\begin{equation}\label{robust_linear_regression_negative_maximum}
\arraycolsep=1.4pt\def\arraystretch{1.3}
\begin{array}{lll}
    \displaystyle{\max_{\btheta} \quad \min_{\bC, \bb}} \quad &  \displaystyle{-\btheta^T \bC \btheta + 2\bb^T \btheta - \lambda \| \btheta\|_2^2}\\
    \st &  - \bC + \bC_{\min} \leq 0, \\
        & \bC - \bC_{\max} \leq 0, \\
        &  - \bb + \bb_{\min} \leq 0, \\
        & \bb - \bb_{\max} \leq 0, \\
        & - \bC \preceq 0
\end{array}
\end{equation}

If we assign $\bA, \bB, \bd, \be, \bH$ to the constraints respectively, then the Lagrangian function can be written as: 
\begin{equation}\label{Lagrangian_Function}
\arraycolsep=1.4pt\def\arraystretch{1.3}
\begin{array}{lll}
    \displaystyle{L(\bC, \bb, \bA, \bB, \bd, \be, \bH)} = \quad &  \displaystyle{-\btheta^T \bC \btheta + 2\bb^T \btheta + \langle \bA, - \bC + \bC_{\min} \rangle}\\
        &+   \langle \bB, \bC - \bC_{\max} \rangle + \langle \bd, - \bb + \bb_{\min} \rangle\\
        &+  \langle \be, \bb - \bb_{\max} \rangle - \langle \bC, \bH \rangle - \lambda \| \btheta\|_2^2,
\end{array}
\end{equation}
The dual problem is defined as:
\begin{equation}\label{General_Dual}
\arraycolsep=1.4pt\def\arraystretch{1.3}
\begin{array}{lll}
    \displaystyle{\max_{\bA, \bB, \bd, \be, \bH} \min_{\bC, \bb} L(\bC, \bb, \bA, \bB, \bd, \be, \bH)} 
\end{array}
\end{equation}
The minimization of $L$ takes the following form:
\begin{equation}\label{L_minimization}
\arraycolsep=1.4pt\def\arraystretch{1.3}
\begin{array}{lll}
    \displaystyle{ \min_{\bC, \bb} } \quad & \langle \bC, -\btheta \btheta^T - \bA + \bB - \bH \rangle  + \langle \bb, 2\btheta - \bd + \be \rangle - \lambda \| \btheta\|_2^2\\
    &-   \langle \bB, \bC_{\max} \rangle + \langle \bA, \bC_{\min} \rangle -  \langle \be, \bb_{\max} \rangle + \langle \bd, \bb_{\min} \rangle, 
\end{array}
\end{equation}

To avoid $-\infty$ value for the above minimization problem, it is required to set $-\btheta \btheta^T - \bA + \bB - \bH$ and $2\btheta - \bd + \be$ to zero. Thus the dual problem of~\eqref{robust_linear_regression_negative_maximum} is formulated as:
\begin{equation}\label{robust_linear_regression_dual}
\arraycolsep=1.4pt\def\arraystretch{1.3}
\begin{array}{lll}
    \displaystyle{\max_{\bA, \bB, \bd, \be, \bH}} \quad &  \displaystyle{\bb_{\min}^T \bd - \bb_{\max}^T \be +  \langle \bC_{\min}, \bA \rangle - \langle \bC_{\max}, \bB \rangle - \lambda \| \btheta\|_2^2}\\
    \st &  - \btheta \btheta^T - \bA + \bB - \bH = 0, \\
        & 2\btheta - \bd + \be = 0, \\
        & \bA, \bB, \bd, \be \geq 0, \\
        & \bH \succeq 0
\end{array}
\end{equation}
Since the duality gap is zero, Problem~\eqref{robust_linear_regression} can be equivalently formulated as:
\begin{equation}\label{robust_linear_regression_dual_complete}
\arraycolsep=1.4pt\def\arraystretch{1.3}
\begin{array}{lll}
    \displaystyle{\max_{\btheta, \bA, \bB, \bd, \be, \bH}} \quad &  \displaystyle{\bb_{\min}^T \bd - \bb_{\max}^T \be +  \langle \bC_{\min}, \bA \rangle - \langle \bC_{\max}, \bB \rangle - \lambda \| \btheta\|^2}\\
    \st &  - \btheta \btheta^T - \bA + \bB - \bH = 0, \\
        & 2\btheta - \bd + \be = 0, \\
        & \bA, \bB, \bd, \be \geq 0, \\
        & \bH \succeq 0.
\end{array}
\end{equation}
We can multiply the objective function by $-1$ and change the maximization to minimization, which gives the dual problem described in~\eqref{eq: robust_linear_regression_dual}. \newline

\noindent \textbf{Proof of Theorem}~\ref{thm: regression_performance}:

\noindent \textbf{(a)} Let $\boldsymbol{\Delta}_n$ be the estimated confidence matrix obtained from $n$ samples. The first part of the theorem is true, if $\boldsymbol{\Delta}_n$ converges to $0$ as $n$, the number of samples goes to infinity (the same argument works for $b$ and $\delta$). Assume that $\{(\bx_{i1}, \bx_{i2})\}_{i=1}^{n}$ is an i.i.d bootstrap sample over data points that both features $\bX_1$ and $\bX_2$ are available. Since the distribution of missing values is completely at random (MCAR), we have $\mathbb{E}[\bx_{i1} \bx_{i2}] = \mathbb{E}[\bX_1 \bX_2]$. Therefore, $\mathbb{E}[\frac{1}{n} \sum_{i=1}^{n} \bx_{i1} \bx_{i2}] = \mathbb{E}[\bX_1 \bX_2]$. 
Moreover, since the samples are drawn independently,  $\textrm{Var}[\frac{1}{n} \sum_{i=1}^{n} \bx_{i1} \bx_{i2}] = \frac{1}{n^2} \sum_{i=1}^n \textrm{Var}[\bx_{i1} \bx_{i2}] = \frac{n}{n^2} \textrm{Var}[\bX_{1} \bX_{2}] = \frac{1}{n} \textrm{Var}[\bX_{1} \bX_{2}]$. Since the variance of the product of every two features is bounded, according to the weak law of large numbers:
\begin{equation*}
    \lim_{n \rightarrow \infty} \textrm{Pr} \Big(|\frac{1}{n} \sum_{i=1}^{n} \bx_{i1} \bx_{i2} - \mathbb{E}[X_1 X_2]| \geq \epsilon \Big) = 0
\end{equation*}
Therefore, for any given bootstrap sample of features $X_1$ and $X_2$, the estimation converges in probability to the ground-truth value. This means the size of the confidence interval $\Delta_{12}$ converges in probability to 0. Therefore, the estimation of $\mathbb{E}[X_1 X_2]$ is consistent by the definition of consistency. With the same argument, we can prove the consistency of the estimator for any given features $X_i$ and $X_j$. 
\\

\noindent \textbf{(b)} Fix two features $i$ and $j$. Let $(\hat{\bX}_{i1}, \hat{\bX}_{j1}), \dots, (\hat{\bX}_{im}, \hat{\bX}_{jm})$ be $m = \lceil n(1-p) \rceil$ i.i.d pairs sampled via bootstrap from the entries where both features $i$ and $j$ are available. Define $Z_{t} = \hat{X}_{it} \hat{X}_{jt}$ (for simplicity we do not consider the dependence of $Z_t$ to $i$ and $j$ in the notation). Assume that we initialize $\bC_{ij} = \frac{1}{m} \sum_{t = 1}^m Z_t$. Note that, $\mathbb{E} [Z_t] = \mathbb{E} [\hat{X}_{it} \hat{X}_{jt}] = \bC^{*}_{ij}$. According to Chebyshev's inequality, we have:
\begin{equation*}
\textrm{Pr} \Bigg[ \quad \Bigg|\frac{1}{m} \sum_{t=1}^{m} \Big(Z_t - \mathbb{E} [Z_t] \Big) \Bigg| \geq \boldsymbol{\Delta}_{ij} \Bigg] \leq \frac{\textrm{Var}(\frac{1}{m} \sum_{t=1}^m Z_t)}{c^2 \boldsymbol{\Delta}_{ij}^2}  
\end{equation*}
Note that $Z_i$'s are iid samples, thus:
\begin{equation*}
\textrm{Var}(\frac{1}{m} \sum_{t=1}^m Z_t) = \frac{1}{m}\textrm{Var}(Z_t) \leq \frac{1}{m}\max_{i, j} \textrm{Var}(\hat{X}_i \hat{X}_j) = \frac{V}{m} = \frac{V}{n(1-p)} 
\end{equation*}
Let $\boldsymbol{\Delta} = \min \{ \boldsymbol{\Delta}_{ij}\}$. Then, based on the two above inequalities, we have:
\begin{equation*}
\textrm{Pr} \Bigg[ \quad \Bigg|\bC_0[i][j] - \bC^{*}[i][j] \Bigg| \geq \boldsymbol{\Delta} \Bigg] \leq \frac{V}{c^2 \boldsymbol{\Delta}^2 n(1-p)} 
\end{equation*}
Using a union bound argument, with the probability of at least $1 - \frac{Vd^2}{2c^2 \boldsymbol{\Delta}^2 n(1-p)}$, we have:
$\bC_0 - c\boldsymbol{\Delta} \leq \bC^* \leq \bC_0 + c\boldsymbol{\Delta}$, which means the actual covariance matrix is within the confidence intervals we have considered. In that case, for $(\tilde{\btheta}, \tilde{\bb}, \tilde{\bC})$, we have:
\begin{gather*}
\tilde{\btheta}^T \tilde{\bC} \tilde{\btheta} - 2 \tilde{\bb}^T \tilde{\btheta} = \max_{\bC, \bb} \quad \tilde{\btheta}^T \bC \tilde{\btheta} - 2 \bb^T \tilde{\btheta} \geq
\tilde{\btheta}^T \bC^* \tilde{\btheta} - 2 \bb^{*T} \tilde{\btheta},
\end{gather*}
which completes the proof.

\newpage
\noindent \textbf{Proof of Theorem}~\ref{thm: mu_max}: 
Since the objective function is \textbf{convex} with respect to $\bmu_1$, and the constraint on $\bmu_1$ is closed and bounded (compact), an optimal solution exists to the problem on the boundaries (note that the problem is convex \textbf{maximization}.) Therefore, for any entry of the $\bmu_1$, it should either take $\bmu_{\min}[i]$
or $\bmu_{\max}[i]$, which gives the provided solution in the theorem.

\vspace{-3mm}
\section{Dataset Descriptions}
\vspace{-3mm}
\label{appendix: datasets}
In this section, we introduce the datasets used in Section~\ref{sec: numerical_results} to evaluate the performance of RIFLE on regression and classification tasks. Except for the NHANES, all other datasets contain no missing values. For those datasets, we generate MCAR and MNAR missing values artificially (for MNAR patterns, we apply Algorithm~\ref{alg: MNAR_Generator} to the datasets).

\subsection*{Datasets for Evaluating RIFLE on Regression and Imputation Tasks}

\begin{itemize}

    \item \textbf{NHANES}: The percentage of missing values varies for different features of the NHANES dataset. There are two sources of missing values in NHANES data: Missing entries during data collection and missing entries resulting from merging different datasets in the NHANES collection. On average, approximately $20\%$ of data is missing.

    \item \textbf{Super Conductivity}\footnote{\url{https://archive.ics.uci.edu/ml/datasets/Superconductivty+Data}}: Super Conductivity datasets contains $21263$ samples describing superconductors and their relevant features ($81$ attributes). All features are continuous, and the assigned task is to predict the critical temperature based on the given $81$ features. We have used this dataset in experiments summarized in Figure~\ref{fig: ADMM}, Figure~\ref{fig: nesterov}, and Table~\ref{tab:real_datasets_mnar}.
    
    \item \textbf{BlogFeedback}\footnote{\url{https://archive.ics.uci.edu/ml/datasets/BlogFeedback}}: BlogFeedback data is a collection of $280$ features extracted from HTML-documents of the blog posts. The assigned task is to predict the number of comments in the upcoming $24$ hours based on the features of more than $60K$ data training data points. The test dataset is fixed and is originally separated from the training data. The dataset is used in experiments described in Table~\ref{tab:real_datasets_mnar}.

    \item \textbf{Breast Cancer}(Prognostic)\footnote{\url{https://archive.ics.uci.edu/ml/datasets/Breast+Cancer+Wisconsin+(Prognostic)}}: The dataset consists of $34$ features and $198$ instances. Each record represents follow-up data for one breast cancer case collected in $1984$. We have done several experiments to impute the MCAR missing values generated artificially with different proportions. The results are depicted in Table~\ref{tab:imputation_uci} and Figure~\ref{fig: counts}.  
    
    \item \textbf{Parkinson}\footnote{\url{https://archive.ics.uci.edu/ml/datasets/parkinsons}}: The dataset describes a range of biomedical voice recording from $31$ people, $23$ with Parkinson's disease (PD). The assigned task is to discriminate healthy people from those with PD. There are $193$ records and $23$ features in the dataset. The dataset is processed similarly to the Breast Cancer dataset and used in the same experiments. 
    
    \item \textbf{Spam Base}\footnote{\url{https://archive.ics.uci.edu/ml/datasets/spambase}}: The dataset consists of $4601$ instances and $57$ attributes. The assigned classification task is to predict whether the email is spam. To evaluate different imputation methods, we randomly mask a proportion of data entries and impute them with different approaches. The results are depicted in  Table~\ref{tab:imputation_uci} and Figure~\ref{fig: counts}.
    
    \item \textbf{Boston Housing}\footnote{\url{https://www.kaggle.com/c/boston-housing}}: Boston Housing dataset contains $506$ instances and $14$ columns. We generate random missing entries with different proportions and impute them with RIFLE and several state-of-the-art approaches.  The results are demonstrated in  Table~\ref{tab:imputation_uci} and Figure~\ref{fig: counts}.
    
    \item \textbf{Cloud}\footnote{\url{https://archive.ics.uci.edu/ml/datasets/Cloud}}: The dataset has $1024$ instances and $10$ features extracted from clouds images. We use this dataset in experiments depicted in Table~\ref{tab:imputation_uci} with $70\%$ artificial MCAR missing values.    
    
    \item \textbf{Wave Energy Converters}\footnote{\url{https://archive.ics.uci.edu/ml/datasets/Wave+Energy+Converters}}: We sample a subset of $3000$  instances with $49$ features from the original Wave Energy Converter dataset. We have executed several imputation methods on the dataset, and the results are shown in Figure~\ref{fig: counts}.
    
    \item \textbf{Sensorless Drive Diagnosis}\footnote{\url{https://archive.ics.uci.edu/ml/datasets/dataset+for+sensorless+drive+diagnosis}}: The $49$ continuous features in this dataset are extracted from electric current drive signals, and the associated classification task is to determine the condition of device's motor. We choose different random samples with size $400$ to run experiments (imputation) in Figure~\ref{fig: sensitivity_missing_value_proportion}. 
\end{itemize}

\subsection*{Datasets for Evaluating Robust QDA on Classification Tasks}

\begin{itemize}
    \item \textbf{Avila}\footnote{\url{https://archive.ics.uci.edu/ml/datasets/Avila}}: The Avila dataset consists of $10$ attributes extracting from $800$ images of "Avila Bible". The associated classification task is to match each pattern (an instance of the dataset) to a copyist. We put $40\%$ of MCAR missing values (both input features and the target variable) on $10$ different random samples of the dataset with size $1000$. The average accuracy of the robust LDA method on the $10$ datasets is demonstrated in Figure~\ref{fig: number_of_estimations} for each value of $k$ (the number of covariance estimations).

    \item \textbf{Magic Gamma Telescope}\footnote{\url{https://archive.ics.uci.edu/ml/datasets/MAGIC+Gamma+Telescope}}: The dataset consists of $11$ continuous MC-generated features from contributing to the prediction of the type of event (signal or background). We used the same procedure as the above dataset for the results depicted in Figure~\ref{fig: number_of_estimations} (random sampling subsets of $1000$ data points out of more than $19000$). 
    
    \item \textbf{Glass Identification}\footnote{\url{https://archive.ics.uci.edu/ml/datasets/glass+identification}}: This dataset is composed of $10$ continuous features and $214$ instances. The assigned classification task is to predict the type of glasses based on the materials used for making them. We have assigned $40\%$ of MCAR missing values to the dataset for the experiments reported in Table~\ref{tab:robust_linear_regression_vs_lda}. 
    
    \item \textbf{Annealing}\footnote{\url{https://archive.ics.uci.edu/ml/datasets/Annealing}} This dataset is a mix of categorical and numerical features ($37$ in total), and the associated task is to predict the class (5 classes) of instances (metals). The number of instances in this dataset is $798$. We use $500$ data points as the training data and the rest as the test. We apply $40\%$ of MCAR missing values to both input features and the target variable. The accuracy of different models is reported in Table~\ref{tab:robust_linear_regression_vs_lda}. 
    
    \item \textbf{Abalone}\footnote{\url{https://archive.ics.uci.edu/ml/datasets/abalone}}: This dataset consists of $4177$ instances and $8$ categorical and continuous features. The goal is to predict the age of abalone based on physical measurements. The first $1000$ samples are used as the training data and the rest as the test data. We applied the same pre-processing procedure as the above dataset to generate missing values on the training data. The accuracy of different models is reported in Table~\ref{tab:robust_linear_regression_vs_lda}.  
    
    \item \textbf{Lymphography}\footnote{\url{https://archive.ics.uci.edu/ml/datasets/Lymphography}}: Lymphography is a categorical dataset containing $18$ features and $148$ data points obtained from the University Medical Centre, Institute of Oncology, Ljubljana, Yugoslavia. $100$ data points are used as the training data; the rest are test data points (with no missing values). We applied the same pre-processing described for the above dataset to generate MCAR missing values. 
    
    \item \textbf{Adult} \footnote{\url{https://archive.ics.uci.edu/ml/datasets/adult}.}: The adult dataset contains census information of individuals, including education, gender, and capital gain. The assigned classification task is to predict whether a person earns over 50k annually. The train and test sets are two separate files consisting of $32,000$ and $16,000$ samples, respectively. We consider gender and race as the sensitive attributes (For the experiments involving one sensitive attribute, we have chosen gender). Learning a logistic regression model on the training dataset (without imposing fairness) shows that only $3$ features out of $14$ have larger weights than the gender attribute.
\end{itemize}

\vspace{-5mm}
\section{Further Discussion on the Consistency of RIFLE}
\vspace{-3mm}
\label{appendix: consistency}
The three developed algorithms in Section~\ref{sec: robust_linear_regression} for solving robust ridge regression are all consistent. To show this, we have generated a synthetic dataset with $50$ input features following a jointly normal distribution. As observed in Figure~\ref{fig: consistency}, by increasing the number of samples, the NRMSE of all three algorithms converges to $0.01$, which is the standard deviation of the zero-mean Gaussian noise added to each target value (the dashed line). The pattern can be observed for different percentages of missing values.
\begin{figure}[H]
\centering
\vspace{-2mm}
\includegraphics[width=0.88\columnwidth]{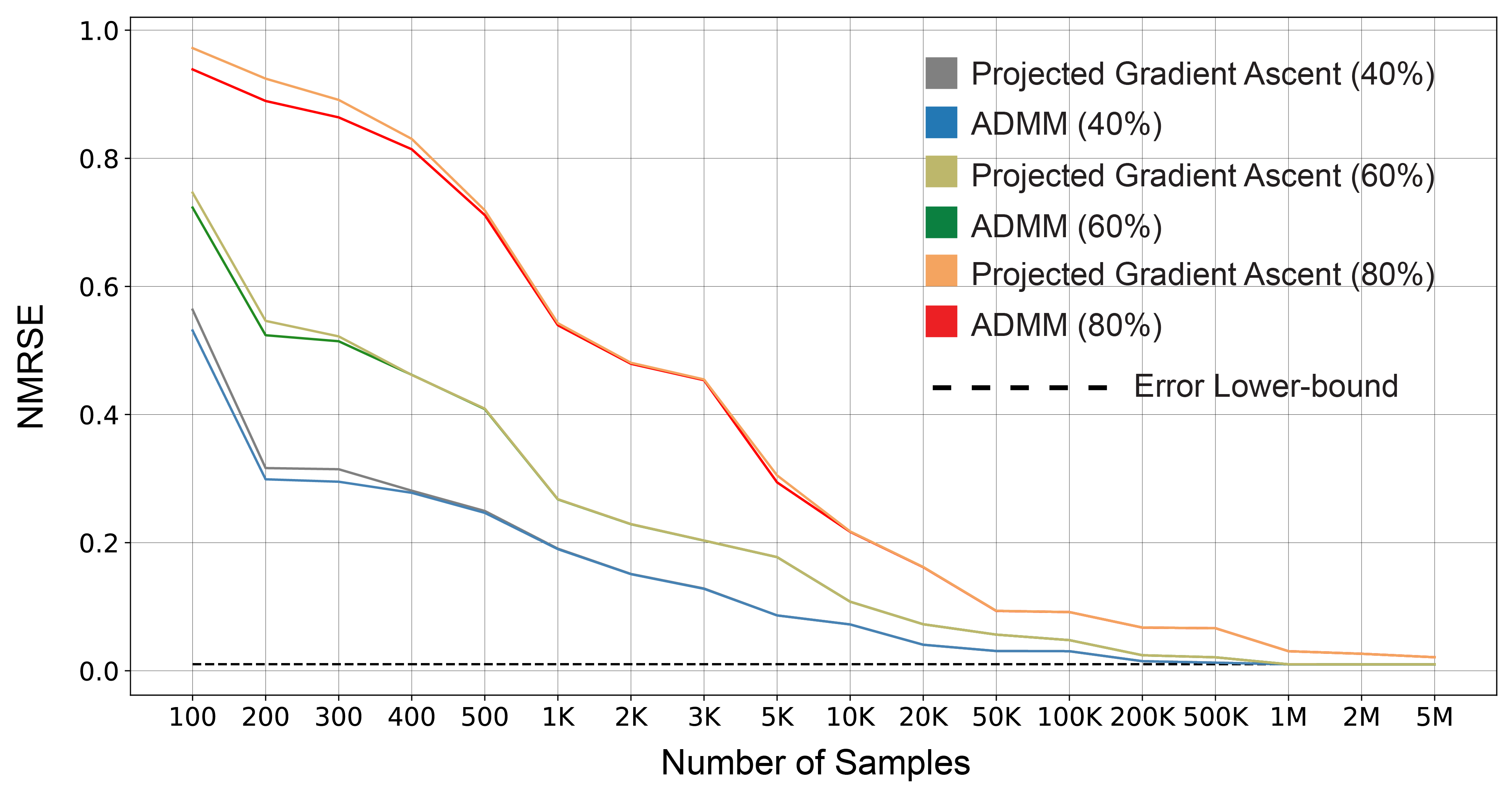}
\vspace{-3mm}
\caption{Consistency of ADMM (Algorithm~\ref{alg: solving_dual_admm}) and Projected Gradient Ascent on function $g$ (Algorithm~\ref{alg: Min_Max_Regression}) on the synthetic datasets with $40\%$, $60\%$ and $80\%$ missing values.}
\label{fig: consistency}
\end{figure}

\vspace{-7mm}
\section{Numerical Experiments for Convergence of RIFLE Algorithms}
\vspace{-2mm}
\label{appendix: convergence}
We presented three algorithms for solving the robust linear regression problem formulated in~\eqref{robust_linear_regression}: Projected gradient ascent (Algorithm~\ref{alg: Min_Max_Regression}, Nesterov acceleration method (Algorithm~\ref{alg: Min_Max_Regression_Nesterov}), and Alternating Direction Method of Multipliers (ADMM) (Algorithm~\ref{alg: solving_dual_admm}) applied on the dual problem. 

\begin{figure}[H]
\centering
\vspace{-2mm}
\includegraphics[width=0.88\columnwidth]{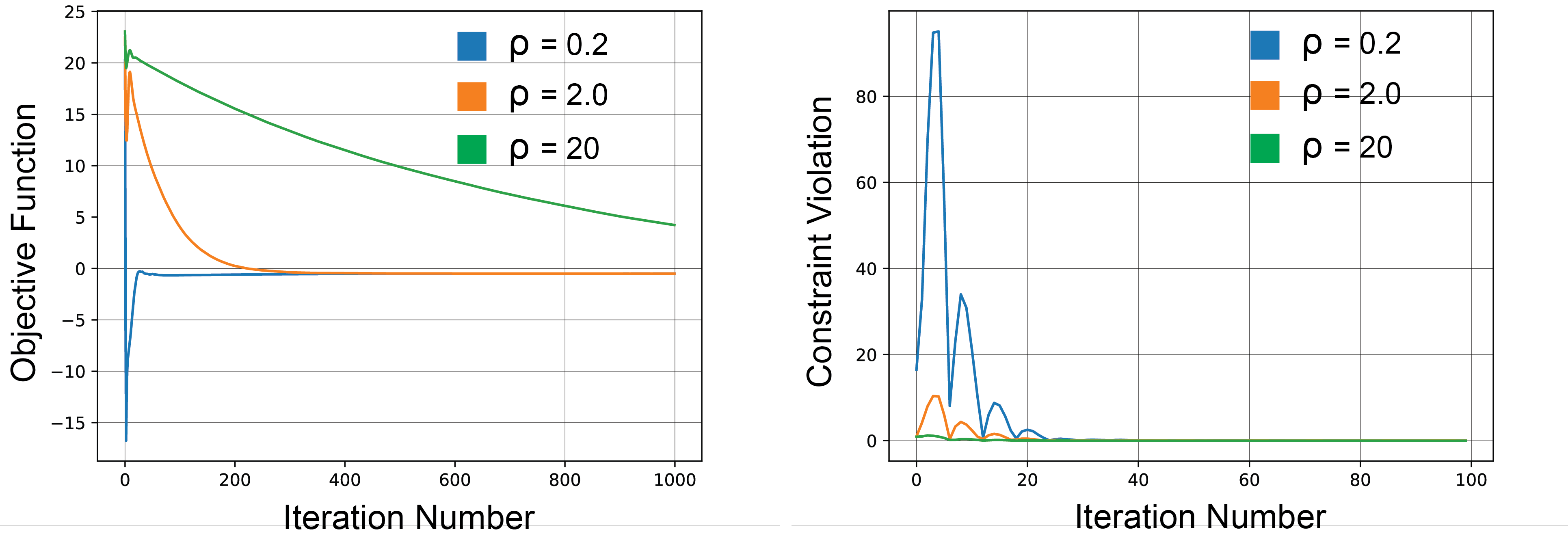}
\vspace{-3mm}
\caption{Convergence of ADMM algorithm to the optimal solution of Problem~\eqref{eq: robust_linear_regression_dual} for different values of $\rho$. The left plot measures the objective function of Problem~\eqref{eq: robust_linear_regression_dual} per iteration (without considering the constraints), while the right plot demonstrates the constraint violation of the algorithm per iteration. The constraint violation can be measured by adding all regularization terms in the augmented Lagrangian function formulated in Problem~\eqref{robust_linear_regression_dual_augmented_lagrangian}.}
\label{fig: ADMM}
\end{figure}
\vspace{-3mm}
We established the convergence rate of the gradient ascent and Nesterov acceleration methods in Theorem~\ref{thm: Robust_Regression_Convergence} and Theorem~\ref{thm: Nesterov_Convergence}, respectively. To investigate the convergence of the ADMM algorithm and its dependence on $\rho$, we perform Algorithm~\ref{alg: solving_dual_admm} on the Super Conductivity dataset (Description in Appendix~\ref{appendix: datasets}) with $30\%$ MCAR missing values. Figure~\ref{fig: ADMM} demonstrates the convergence of the ADMM algorithm for multiple values of $\rho$ applied to the Super Conductivity dataset as described above. As can be observed, decreasing the value of $\rho$ accelerates the ADMM convergence to the optimal value. Note that for $\rho = 0.2$, the objective function is smaller than the final value in the first few iterations. The reason is that for those iterations, the solution is not feasible (as observed in the right figure). The final solution is the optimal \textbf{feasible} solution.

In the next experiment, we compare the three proposed algorithms regarding the number of iterations required to reach a certain level of test accuracy on the Super Conductivity dataset. The number of training samples is $1000$, containing $40\%$ of MCAR missing values on both input features and the target variable. The test dataset contains $2000$ samples. As depicted in Figure~\ref{fig: nesterov}, ADMM and Nesterov's algorithms require less number of iterations to reach the $\epsilon$-optimal solution compared to Algorithm~\ref{alg: Min_Max_Regression}. However, the cost per iteration of the ADMM algorithm (Algorithm~\ref{alg: solving_dual_admm}) is higher than the Nesterov acceleration and Algorithm~\ref{alg: Min_Max_Regression}. 
\begin{figure}[ht]
\centering
\includegraphics[width=0.9\columnwidth]{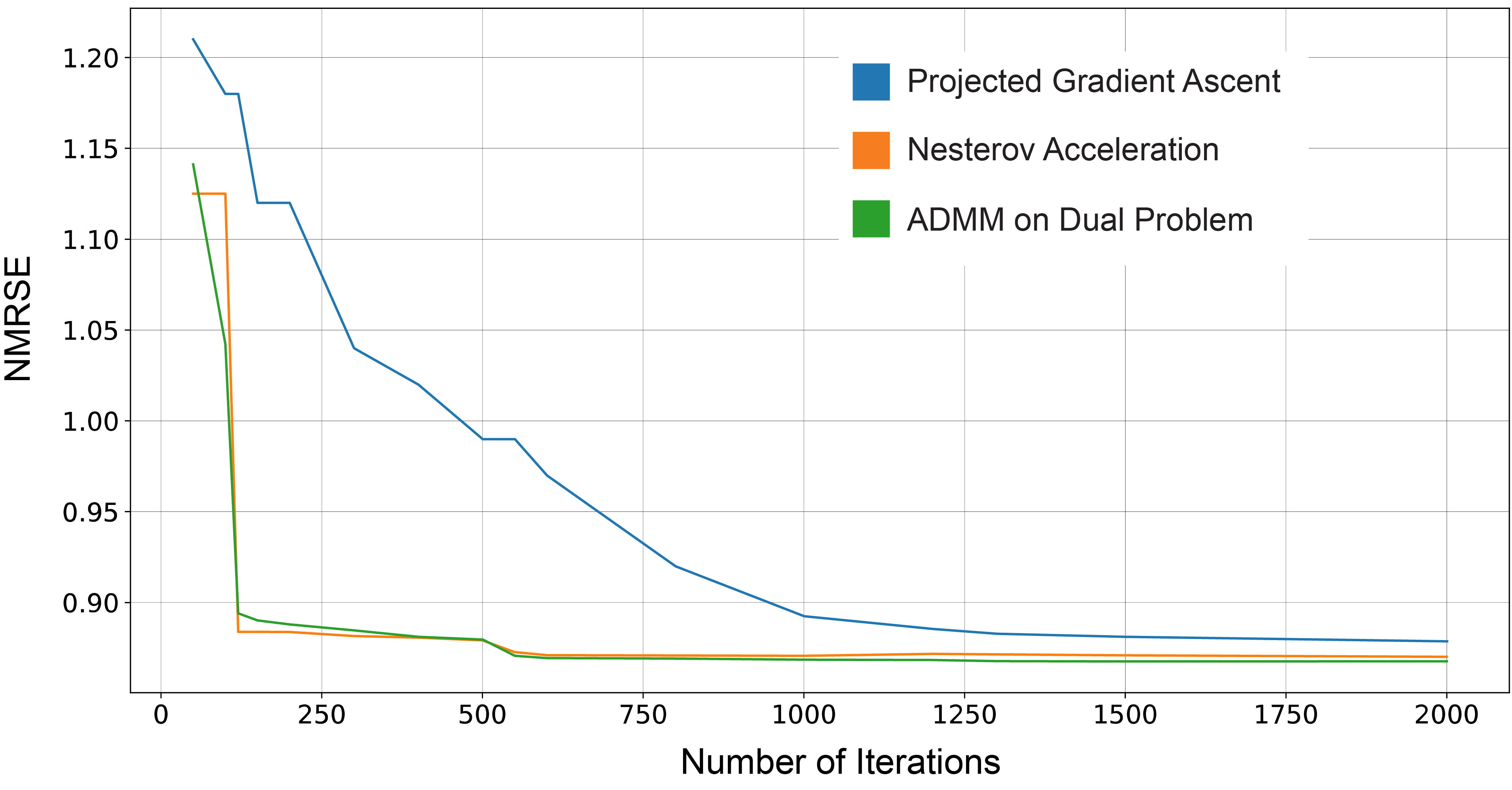}
\vspace{-5mm}
\caption{The performance of the Nesterov acceleration method, projected gradient ascent, and ADMM on the Super Conductivity dataset vs. the number of iterations.}
\label{fig: nesterov}
\end{figure}

\vspace{-3mm}
\section{Execution Time Comparison of RIFLE and Other State-of-the-art Approaches}
\label{appendix: time}
This section reports the average execution time of the RIFLE and other approaches presented in Table~\ref{tab:real_datasets_mnar}.
\begin{table}[ht]
\begin{center}
\begin{tabular}{|c|c|c|c|c}
\hline
\multirow{2}{*}{\textbf{Methods}} & \multicolumn{3}{c|}{\textbf{Datasets}}      \\ \cline{2-4} 
                                  & Super Conductivity & Blog Feedback & NHANES \\ \hline
Regression on Complete Data       & $0.3$ sec             & $0.7$ sec        & $0.4$ sec \\ \hline
\textbf{RIFLE}                             & $87$ sec             & $471$ sec        & $125$ sec \\ \hline
Mean Imputer + Regression      & $0.4$ sec             & $0.9$ sec        & $0.5$ sec\\ \hline
MICE + Regression      & $112$ sec             & $573$ sec        & $294$ sec \\ \hline
EM + Regression  & $171$ sec  & $612$ sec & $351$ sec \\ \hline
MIDA Imputer + Regression   & $245$ sec & $726$ sec  & $599$ sec\\ \hline
MissForest   & $94$ sec & $321$ sec        & $132$ sec  \\ \hline
KNN Imputer  & $66$ sec             & $292$ sec        & $144$ sec \\ \hline
\end{tabular}    
\end{center}
\caption{Execution time of RIFLE and other SOTA methods on three datasets.}
\label{tab:time} 
\end{table}

\end{document}